\documentclass[11pt]{llncs}  

\usepackage[utf8]{inputenc}
\usepackage[english]{babel}
\usepackage{graphicx}
\usepackage{float}
\usepackage{amssymb}
\usepackage{amsmath}
\usepackage{mathtools}
\usepackage{xcolor}
\usepackage{xspace}
\usepackage{algorithm}
\usepackage{multicol}
\usepackage{varwidth}
\usepackage{wrapfig}
\usepackage[noend]{algpseudocode}
\usepackage{caption}
\usepackage{marginnote}
\usepackage{multicol}

\usepackage{hyperref}
\usepackage{amssymb,amsmath}
\usepackage{subfigure}
\newtheorem{assumption}{Assumption}

\usepackage{paralist}
\usepackage{xspace}

\usepackage[normalem]{ulem}

\newcommand{\num}[1]{\relax\ifmmode \mathbb #1\else $\mathbb #1$\fi}
\newcommand{\nnnum}[1]{\relax\ifmmode 
  {\mathbb #1}_{\geq 0} \else ${\mathbb #1}_{\geq 0}$
  \fi}
\newcommand{\npnum}[1]{\relax\ifmmode 
  {\mathbb #1}_{\leq 0} \else ${\mathbb #1}_{\leq 0}$
  \fi}
\newcommand{\pnum}[1]{\relax\ifmmode 
  {\mathbb #1}_{> 0} \else ${\mathbb #1}_{> 0}$
  \fi}
\newcommand{\nnum}[1]{\relax\ifmmode 
  {\mathbb #1}_{< 0} \else ${\mathbb #1}_{< 0}$
  \fi}
\newcommand{\plnum}[1]{\relax\ifmmode 
  {\mathbb #1}_{+} \else ${\mathbb #1}_{+}$
  \fi}
\newcommand{\nenum}[1]{\relax\ifmmode 
  {\mathbb #1}_{-} \else ${\mathbb #1}_{-}$
  \fi}

\newcommand{\reals}{{\num R}}                    
    
\newcommand{\M}{\mathcal{M}}
\newcommand{\X}{\mathcal{X}}

\newcommand{\F}{\mathcal{F}}

\newcommand{\phit}[3]{{p_{#1,#2}{(#3)}}}
\newcommand{\Expectation}[1]{\mathbb{E}\left[#1\right]}

\newcommand{\observ}{y}
\newcommand{\Observ}{y}

\newcommand{\Unsafe}{\mathcal{U}}

\newcommand{\modelname}{NMC\xspace}

\newcommand{\partition}[2]{{\mathcal{P}_{#1,#2}}}
\newcommand{\counter}[2]{{\mathit{count}_{#1,#2}}}

\newcommand{\batch}{{\mathit{b}}}

\newcommand{\Mlplatoon}{{$\mathsf{MLplatoon}$\xspace}}
\newcommand{\Merging}{{$\mathsf{Merging}$\xspace}}
\newcommand{\DetectingPedestrian}{{$\mathsf{DetectBrake}$\xspace}}

\newcommand{\toolname}{{{\sf HooVer}\xspace}}

\usepackage{listings}

\usepackage{stmaryrd}
\usepackage{multirow}
\usepackage{array}
\usepackage{upgreek}

\newcommand{\two}[4]{
  \parbox{.95\columnwidth}{\vspace{1pt} \vfill
    \parbox[t]{#1\columnwidth}{#3}%
    \parbox[t]{#2\columnwidth}{#4}%
  }}
  
\lstdefinelanguage{pseudocode}{
	basicstyle=\scriptsize,
	keywordstyle=\bf \scriptsize,
	identifierstyle=\it \scriptsize,
	mathescape=true,
	tabsize=20,
	xleftmargin=4.0ex,
	sensitive=false,
	columns=fullflexible,
	keepspaces=false,
	basewidth=0.05em,
	moredelim=[il][\rm]{//},
	moredelim=[is][\sf \figuresize]{!}{!},
	moredelim=[is][\bf \figuresize]{*}{*},
	keywords={automaton, algorithm, and, 
		break,
		choose,const,continue, components,
		discrete, do,
		eff, external,else, elseif, evolve, end, each, exit,
		fi,for, forward, from, find, 
		hidden,
		in,input,internal,if,invariant, initially, imports,
		let,
		mode,
		or, output, operators, od, of,
		pre,
		return,
		such,satisfies, stop, signature, simulation, sample,
		trajectories,trajdef, transitions, that,then, type, types, to, tasks,
		variables, vocabulary, 
		when,where, with,while},
	emph={set, seq, tuple, map, array, enumeration},   
	literate=
	{(}{{$($}}1
	{)}{{$)$}}1
	{\\in}{{$\in\ $}}1
	{\\preceq}{{$\preceq\ $}}1
	{\\subset}{{$\subset\ $}}1
	{\\subseteq}{{$\subseteq\ $}}1
	{\\supset}{{$\supset\ $}}1
	{\\supseteq}{{$\supseteq\ $}}1
	{\\forall}{{$\forall$}}1
	{\\le}{{$\le\ $}}1
	{\\ge}{{$\ge\ $}}1
	{\\gets}{{$\gets\ $}}1
	{\\cup}{{$\cup\ $}}1
	{\\cap}{{$\cap\ $}}1
	{\\langle}{{$\langle$}}1
	{\\rangle}{{$\rangle$}}1
	{\\exists}{{$\exists\ $}}1
	{\\bot}{{$\bot$}}1
	{\\rip}{{$\rip$}}1
	{\\emptyset}{{$\emptyset$}}1
	{\\notin}{{$\notin\ $}}1
	{\\not\\exists}{{$\not\exists\ $}}1
	{\\ne}{{$\ne\ $}}1
	{\\to}{{$\to\ $}}1
	{\\implies}{{$\implies\ $}}1
	{<}{{$<\ $}}1
	{>}{{$>\ $}}1
	{=}{{$=\ $}}1
	{~}{{$\neg\ $}}1
	{|}{{$\mid$}}1
	{'}{{$^\prime$}}1
	{\\A}{{$\forall\ $}}1
	{\\E}{{$\exists\ $}}1
	{\\/}{{$\vee\,$}}1
	{\\vee}{{$\vee\,$}}1
	{/\\}{{$\wedge\,$}}1
	{\\wedge}{{$\wedge\,$}}1
	{=>}{{$\Rightarrow\ $}}1
	{->}{{$\rightarrow\ $}}1
	{<=}{{$\Leftarrow\ $}}1
	{<-}{{$\leftarrow\ $}}1
	{~=}{{$\neq\ $}}1
	{\\U}{{$\cup\ $}}1
	{\\I}{{$\cap\ $}}1
	{|-}{{$\vdash\ $}}1
	{-|}{{$\dashv\ $}}1
	{<<}{{$\ll\ $}}2
	{>>}{{$\gg\ $}}2
	{||}{{$\|$}}1
	{[}{{$[$}}1
	{]}{{$\,]$}}1
	{[[}{{$\langle$}}1
	{]]]}{{$]\rangle$}}1
	{]]}{{$\rangle$}}1
	{<=>}{{$\Leftrightarrow\ $}}2
	{<->}{{$\leftrightarrow\ $}}2
	{(+)}{{$\oplus\ $}}1
	{(-)}{{$\ominus\ $}}1
	{_i}{{$_{i}$}}1
	{_j}{{$_{j}$}}1
	{_{i,j}}{{$_{i,j}$}}3
	{_{j,i}}{{$_{j,i}$}}3
	{_0}{{$_0$}}1
	{_1}{{$_1$}}1
	{_2}{{$_2$}}1
	{_n}{{$_n$}}1
	{_p}{{$_p$}}1
	{_k}{{$_n$}}1
	{-}{{$\ms{-}$}}1
	{@}{{}}0
	{\\delta}{{$\delta$}}1
	{\\R}{{$\R$}}1
	{\\Rplus}{{$\Rplus$}}1
	{\\N}{{$\N$}}1
	{\\times}{{$\times\ $}}1
	{\\tau}{{$\tau$}}1
	{\\alpha}{{$\alpha$}}1
	{\\beta}{{$\beta$}}1
	{\\gamma}{{$\gamma$}}1
	{\\ell}{{$\ell\ $}}1
	{\\TT}{{\hspace{1.5em}}}3        
}

\lstdefinelanguage{pseudocodeNums}[]{pseudocode}
{
	numbers=left,
	numberstyle=\tiny,
	stepnumber=2,
	numbersep=4pt
}

\definecolor{codegreen}{rgb}{0,0.6,0}
\definecolor{codegray}{rgb}{0.5,0.5,0.5}
\definecolor{codepurple}{rgb}{0.58,0,0.82}
\definecolor{backcolour}{rgb}{0.95,0.95,0.92}

\lstdefinestyle{mystyle}{
  backgroundcolor=\color{backcolour},   commentstyle=\color{codegreen},
  keywordstyle=\color{magenta},
  numberstyle=\tiny\color{codegray},
  stringstyle=\color{codepurple},
  basicstyle=\ttfamily\scriptsize,
  breakatwhitespace=false,         
  breaklines=true,                 
  captionpos=b,                    
  keepspaces=true,                 
  numbers=left,                    
  numbersep=5pt,                  
  showspaces=false,                
  showstringspaces=false,
  showtabs=false,                  
  tabsize=2
}

\title{Verification and Parameter Synthesis for Stochastic Systems using  Optimistic Optimization}

\author{Negin Musavi\inst{1}, Dawei Sun\inst{1}, Sayan Mitra\inst{1}, \\ Geir Dullerud\inst{1}, and Sanjay Shakkottai\inst{2}}
\institute{\email{\{nmusavi2,daweis2,mitras,dullerud\}@illinois.edu} \\ ${}^1$University of Illinois at Urbana Champaign \\
\email{sanjay.shakkottai@utexas.edu}
\\ ${}^2$University of Texas  at Austin}

\titlerunning{Verification and Parameter Synthesis for Stochastic Systems using  Optimistic Optimization}
\authorrunning{Musavi, Sun, Mitra, Dullerud, and Shakkottai}

\begin{document}
\maketitle

\begin{abstract}
We present  an algorithm for formal verification and parameter synthesis of continuous state space Markov chains.
%
This class of problems capture design and analysis of a wide variety of autonomous and cyber-physical systems defined by nonlinear and black-box modules.
In order to solve these problems, one has to maximize certain probabilistic objective functions over all choices of initial states and parameters. 
In this paper, we identify the assumptions that make it possible to view this problem as a multi-armed bandit problem.
Based on this fresh perspective, we propose an algorithm (HOO-MB) for solving the problem that carefully instantiates an existing bandit algorithm---Hierarchical Optimistic Optimization---with appropriate parameters.
As a consequence, we obtain theoretical regret bounds on sample efficiency of our solution that depend on key problem parameters like  smoothness, near-optimality dimension, and  batch size.
The batch size parameter enables us to strike a balance between the sample efficiency and the memory usage of the algorithm.
Our experiments, using the tool \toolname{}, suggest that the approach scales to realistic-sized problems and is often more sample-efficient compared to PlasmaLab---a leading tool for verification of stochastic systems. 
Specifically, \toolname{} has distinct advantages in analyzing models in which the objective function has sharp slopes.
In addition, \toolname{} shows promising behaviour in parameter synthesis for a linear quadratic regulator (LQR) example.
\end{abstract}
\section{Introduction}
\label{sec:intro}
The  multi-armed bandit problem  is an idealized model for sequential decision making in unknown random environments. It is frequently used to study exploration-exploitation trade-offs~\cite{thompson1935theory,thompson1933likelihood,robbins1952some}. Significant advances have been made in the last decade,  and  connections have been drawn with optimization, auctions, and online learning~\cite{munos:hal-2014,Bubeck:2011,Bubeck12}.
Recently developed bandit algorithms for  {\em black-box optimization\/}~\cite{munos:hal-2014} can strike a balance between exploiting the most promising parts of the function's domain, and exploring the uncertain parts, and find nearest to optimal solutions for a given sampling budget.
%
%

In this paper, we present a new verification and parameter synthesis algorithm that uses Bandit algorithms for discrete time continuous state-space stochastic models. There have been several works that address verification and parameter synthesis of stochastic systems that require some strong assumptions on the system dynamics~\cite{abate2008probabilistic,tkachev2011infinite,esmaeil2013adaptive}. Other several works such as~\cite{PrajnaJ04,jagtap2020formal} address this problems using barrier functions. It is noted that these works are based on the knowledge on the probability transition between the state which may not be applicable in realistic scenarios. The notable methods related to this class include MODEST for networks of probabilistic automata~\cite{hartmanns2009modest},
PlasmaLab~\cite{Boyer:2013:PFD}, the reinforcement learning-based algorithm of~\cite{henriques2012statistical,david2014time} implemented in PRISM~\cite{HKNP06} and UPPAAL. Approaches for verifying Markov decision process (MDP) with restricted types of schedulers are presented in~\cite{lassaigne2015approximate,hartmanns2014modest,budde2018statistical}.

Here we focus on the verification and parameter synthesis of discrete time Markov chains (MC) over continuous state spaces and uncountable sets of initial states or parameters, where the exact knowledge of probabilistic evolution of states are not required. In other words, we want to find what choices of initial states or parameters maximize certain probabilistic objective functions over all choices of initial states or parameters, without having exact knowledge of the system dynamics.

%

By building this connection with the Bandit literature, we aim to gain not only new algorithms for verification and parameter synthesis, but also new types of bounds on the sample efficiency. 
We propose a tree-based algorithm for solving this problem using a modification of  {\em hierarchical optimistic optimization (HOO)}~\cite{bubeck2011x}. The basic HOO algorithm relies on the well-known   {\em upper confidence bound (UCB)\/} approach for searching the domain of a black-box function~\cite{lai1985asymptotically}.
Our algorithm, called HOO-MB (Algorithm~\ref{Al:HOO_batch}), improves HOO in two significant ways: (1) it relaxes the requirement of an exact semi-metric that captures the smoothness of $f$, but instead, it works with two smoothness parameters~\cite{grill2015black,sen2019noisy}, 
 (2) It takes advantage of batched samples to reduce the impact of the variance from the noisy samples while keeping the tree size small.
We obtain regret bounds (Theorem~\ref{Th:regret_HOO_batch}) on the difference between  the maximum value $f(\bar{x}_N)$ computed by HOO-MB and  the actual maximum $f(x^*)$, as a function of the sampling budget $N$,  the smoothness parameters, and the {\em near-optimality dimension} of $f$. 
%
We also show empirically that batched simulations can help dramatically reduce the number of nodes in the tree and therefore reduce the  running time and memory usage.

We have implemented HOO-MB in a open source tool called \toolname{}\footnote{The tool and all the benchmarks are available from \url{https://www.daweisun.me/hoover/}.}.  The user only has to provide a \texttt{Python} class specifying the transition kernel, the unsafe set, parameter uncertainty, and the initial uncertainty. There is no requirement for learning  a new  language.
We have created a suite of benchmarks models capturing scenarios that are used for certification of autonomous vehicles and advanced driving assist systems~\cite{ISO26262} and evaluated  \toolname{} on these benchmarks.
Our evaluations suggest that
(1) as expected, the quality of the verification result (in this case, maximum probability of hitting an unsafe state) improves with the sampling budget $N$, (2) \toolname{}  scales to reasonably large models. In our experiments, the tool easily handled  models with $18$-dimensional state spaces, and initial uncertainty spanning $8$ dimensions, on a standard computer.
(3) Running time and memory usage can be controlled with the sampling batch sizes, and 
(4) \toolname{} is relatively insensitive to the smoothness parameters.
A  fair comparison of \toolname{} with other discrete state verification approaches is complicated because the guarantees are different and platform specific constants are difficult to factor out. We present a careful comparison with PlasmaLab in Section~\ref{sec:eval-mfhoo}. \toolname{} generally gets closer to the correct answer with  fewer samples than  PlasmaLab. For models with sharp slopes around the maxima, \toolname{} is more sample efficient. 
This suggests that HOO-MB may work with fewer samples in verification problems around hard to find bugs. Our preliminary results were presented in the workshop~\cite{SNR2020Program}.
We have also evaluated the performance of \toolname{} for parameter synthesis for a linear quadratic regulator (LQR) example, and the results shows the efficiency \toolname{} in tuning controller parameters.   
A related approach~\cite{ellen2015statistical}  uses the original HOO algorithm of~\cite{bubeck2011x} for model checking over continuous domains\footnote{We were unable to get this tool and run experiments.}. We believe that requiring upper bounds on the smoothness parameters of $f$, as required in HOO-MB, is  a less stringent requirement than the semi-metric used in~\cite{ellen2015statistical}.
\section{Model and problem statement}
\label{sec:prelims}

Let pair $(\X, \F_\X)$ be a {\em measurable space\/}, where $\F_\X$ is a $\sigma$-algebra over $\X$ and the elements of $\F_\X$ are referred to as {\em measurable sets}. Let $\mathbb{P}:\X \times \F_\X \rightarrow [0, 1]$ be a {\em Markovian transition kernel} on a measurable space $(\X, \F_\X)$, such that \begin{inparaenum}[(i)] \item for all $x \in \X$, $\mathbb{P}(x, \cdot)$ is a probability measure on $\F_\X$; and \item for all $\mathcal{A} \in \F_\X$, $\mathbb{P}(\cdot, \mathcal{A})$ is a $\F_\X$-measurable function. \end{inparaenum} Also let $\mathbb{P}_\beta$ be a Markovian transition kernel that depends on parameter $\beta \in \mathbb{R}^m$. A real-valued random variable\footnote{Namely, a function $X:\X\rightarrow \mathbb{R}$
for which the set $\{x\in \X: X(x)\leq r \}$  is measurable, for each $r\in \mathbb{R}$, with respect to a fixed designated measurable space $(X,\F_\X)$ equipped with a probability measure.}  $X$ is {\em $\sigma^2$-sub-Gaussian\/}, if for all $s \in \mathbb{R}$, $\mathbb{E}\left[\exp(s(X-\mathbb{E}X))\right] \leq \exp(\sigma^2 s^2/2)$ holds, where $\mathbb{E}$ denotes the expectation.


\begin{definition}
\label{def:mpis}
A {\em Nondeterministic Markov chain (\modelname)\/} $\M$ is defined by a quadruplet
$((\X, \F_\X), \mathbb{P}_{\beta}, \mathcal{B}, \Theta)$, with:
\begin{inparaenum}[(i)]
\item $(\X, \F_\X)$, a measurable space over the state space $\X$;
\item $\mathbb{P}_{\beta}:\X \times \F_\X \rightarrow [0, 1]$, a Markovian {\em transition kernel} depending on parameter $\beta \in \mathcal{B}$.
\item $\mathcal{B} \subseteq \mathbb{R}^m$, a set of possible parameters; and
\item $\Theta \subseteq \X$, the set of possible initial states.
\end{inparaenum}
\end{definition}

If the uncertainty is in the initial states, then the uncertainty in the initial states is modeled as a nondeterministic choice over the set $\Theta$ and the set $\mathcal{B}$ is a singleton. If the uncertainty is in the parameters, then the uncertainty in the parameters is modeled as a nondeterministic choice over the  set $\mathcal{B}$ and the set $\Theta$ is a singleton. The uncertainty in the  transitions is modeled as probabilistic choices by the transition kernel $\mathbb{P}_{\beta}$. We address two class of problems:

 \paragraph{Verification}
An {\em execution} of $\M$  of length $k$ is a sequence of states $\alpha = x_0 x_1 \cdots x_k$, where $x_0 \in \Theta$ and for all $i$, $x_i \in \X$. Given $x_0$ and a sequence of measurable sets of states $A_1,\ldots, A_k \in \F_\X$, the measure of the set of executions $\{ \alpha \ | \ \alpha_0 = x_0 \text{ and }  \alpha_i \in A_i, \forall \ i =1,\ldots, k\}$ is given by : 
\begin{align*}
\Pr(\{ \alpha \ | \ \alpha_0 = x_0 \text{ and } \alpha_i \in A_i, \forall \ i =1,\ldots, k \}) &\\ 
= \int_{A_1 \times \cdots \times A_k} \mathbb{P}_{\beta}(x_0, dx_1) \cdots \mathbb{P}_{\beta}(x_{k-1}, dx_k),
\end{align*}
which is a standard result and follows from the Ionescu Tulce{\u a} theorem \cite{ionescu1949mesures}\cite{petritis2012}.

Given an $\modelname$ $\M$ and a measurable unsafe set $\Unsafe \in \F_\X$, we are interested in evaluating the {\em worst case\/} probability of $\M$ hitting $\Unsafe$ over all possible nondeterministic choices of an initial state $x_0$ in $\Theta$. Once an initial state $x_0 \in \Theta$ is fixed, the probability of a set of paths is defined in the standard way. The details of the construction of the measure space over executions is not relevant for our work, and therefore, we give an abridged overview below.

We say that an execution $\alpha$ of length $k$ \emph{hits} the unsafe set $\Unsafe$ if there exists $i\in\{0,\ldots,k\}$, such that  $\alpha_i \in \Unsafe$. The complement of $\Unsafe$, the  {\em safe subset} of $\X$, is denoted by $\mathcal{S}$. The safe set is also a member of the $\sigma$-algebra $\F_\X$ since $\sigma$-algebras are closed under complementation. From a given initial state $x_0 \in \Theta$, the probability of $\M$ hitting $\Unsafe$ within $k$ steps is denoted by $\phit{k}{\Unsafe}{x_0}$. By definition, $\phit{k}{\Unsafe}{x_0} = 1$, if $x_0 \in \Unsafe$. For $x_0 \notin \Unsafe$ and $k \geq 1$,   
\begin{equation}
\label{eq:prob_int}
\phit{k}{\Unsafe}{x_0} = 1 - \int_{S \times \cdots \times S} \mathbb{P}_{\beta}(x_0, dx_1) \cdots \mathbb{P}_{\beta}(x_{k-1}, dx_k).
\end{equation}

We are interested in finding the {\em worst case} probability of hitting unsafe states over all possible initial states of $\M$. This can be regarded as solving, for some $k$, the following optimization problem:
\begin{align}
    \label{eq:SMCproblem} \sup\limits_{x_0 \in \Theta} \phit{k}{\Unsafe}{x_0}.
\end{align}

\paragraph{Parameter Synthesis}
Given an execution $\alpha$ of length $k$, let $r(\alpha, \beta)$ be a real-valued objective function. Then, we are interested in evaluating the maximum of expected objective function over all possible nondeterministic choices of the parameter $\beta$ in $\mathcal{B}$. This can be regarded as solving, for some $k$, the following optimization problem:
\begin{align}
    \label{eq:param_problem} \sup\limits_{\beta \in \mathcal{B}} \mathbb{E}[r(\alpha, \beta) | \beta],
\end{align}
where the expectation is over the stochasticity of the  transition and the initial state which is drawn from a given distribution.

We  note here that our approach solves the optimization problem of Equations~(\ref{eq:SMCproblem}) and~(\ref{eq:param_problem}) using samples of individual elements from $\Theta$ or $\mathcal{B}$, respectively.
For instance, for verification problem our approach does not rely on explicitly calculating the probability defined by Equation~(\ref{eq:prob_int}). Instead it relies on noisy observations about whether or not a sampled execution hits $\Unsafe$. Thus, the user only has to provide a {\em simulator} for the $\modelname$ $\M$ (i.e., the transition kernel $\mathbb{P}_{\beta}$), the parameter set $\mathcal{B}$, the initial set $\Theta$, and the unsafe set $\Unsafe$.

\subsection{A simple example}
\label{ex:simple}
Consider a \modelname of a particle moving randomly on a  plane. 
The  model \texttt{RandomMotion}  is specified as a Python class which is a subclass of \modelname for verification.
The initial set $\Theta \subseteq \reals^2$, specified by \texttt{set\_Theta()},  defines $\Theta = \{(x_1, x_2)~|~x_1 \in [1,2],\,x_2 \in [2,3]\}$. 
The unsafe set $\Unsafe$,  specified by the member function \texttt{is\_unsafe()}, 
defines $\Unsafe = \{(x_1, x_2)~|~x_1^2+x_2^2 > 4\}$. 
The \texttt{transition()} 
function describes the transition kernel $\mathbb{P}_{{\beta}}$. Given an input (pre)state $x$, \texttt{transition()} returns the post-state $x'$ of the transition by sampling  the measure $\mathbb{P}_{\beta}(x, \cdot)$. For this example, $x'$ is computed by adding \texttt{inc} to $x$ where \texttt{inc} is sampled from a $2$-dimensional Gaussian distribution with mean $\mu = (0, 0)$ 
and covariance $\Sigma = \begin{bmatrix} \sigma^2 & 0\\ 0 & \sigma^2 \end{bmatrix}$.
We can write the  transition kernel explicitly by a density function: $ p(x')= \frac{1}{2\pi \sigma^2} \exp(-\frac{1}{2 \sigma^2}(x_1'-x_1)^2+(x_2'-x_2)^2)$. This function is continuous in $x$, and  therefore,  $\phit{k}{\Unsafe}{x_0}$ is also continuous over $\Theta$. For more complicated models, for example, the ones studied in Section \ref{sec:allbenchmarks},  \toolname{} does  not rely on explicit transition kernels, but only the  \texttt{transition()} function for sampling $\mathbb{P}_{\beta}(x, \cdot)$.
\begin{figure}
	\centering
	\hrule
	\two{.46}{.54}
	{
\begin{lstlisting}[language=Python,lastline=11]
class RandomMotion(NiMC):
  def __init__(self, sigma):
    self.set_Theta([[1,2],[2,3]])
    self.sigma = sigma
    
  def is_unsafe(self, state):
    # return True if  |state|> 4
    if np.linalg.norm(state) > 4:
      return True 
    return False 
\end{lstlisting}}
	{
\begin{lstlisting}[language=Python,firstline=11,lastline=20,firstnumber=11]


 def transition(self, state):
   # inc 2-d standard normal(0,sigma^2)
   inc=self.sigma*np.random.randn(2)
   state += inc
   return state # return the new state


# end of class
\end{lstlisting}}
	\hrule
	\caption{\small \toolname{} model description file for a simple random motion model.}
	\label{fig:brownian}
\end{figure}

\section{Verification and Parameter Synthesis with Hierarchical Optimistic Optimization}
\label{sec:multi-fid-background}

We will solve the optimization problems of~(\ref{eq:SMCproblem}) and~(\ref{eq:param_problem}) using the {\em Hierarchical Optimistic Optimization algorithm with Mini-batches (HOO-MB)}. This is a variant of the {\em Hierarchical Optimistic Optimization (HOO)} algorithm~\cite{Bubeck:2011} from the  {\em multi-armed bandits} literature~\cite{munos:hal-2014,Bubeck12,sen2019noisy}. The setup is as follows: suppose we have a sampling budget of $N$ and want to maximize the function $f:\X \rightarrow \mathbb{R}$, which is assumed to have a unique global maximum  that achieves the value $f^* = \underset{x\in\X}{\sup}$ $f(x)$.
The algorithm gets to choose a sequence of sample points (arms) $x_1, x_2, \ldots x_N \in \X$, for which it receives the corresponding sequence of noisy {\em observations\/} (or {\em  rewards\/})
$\observ_{1}, \observ_{2}, \ldots, \observ_N$.
When the sampling budget $N$ is exhausted, the algorithm has to decide the optimal point $\bar{x}_{N} \in \X$ with the aim of minimizing  {\em  regret\/}, which is defined as: 
\begin{equation}\label{eq:simple}
    S_N = f^*-f(\bar{x}_{N}).
\end{equation}
We are interested in algorithms that have two key properties:
(I) The algorithm should be {\em adaptive\/} in the sense that the $(i+1)^{\mbox{st}}$ sample $x_{i+1}$ should depend on the previous samples and outputs; and  
(II) The algorithm should not rely on detailed knowledge of $f$ but {\em only} on the sampled noisy outputs. These algorithms are called {\em black-box} or {\em zeroth-order} algorithms. In order to derive rigorous bounds on the regret, however, we will need to make some light assumptions on the smoothness of $f$ (see Assumption~\ref{Ass:HOO}) and on the relationship between $f(x_i)$ and the corresponding observation $\observ_i$. 
Assumption~\ref{ass:observe} formalizes the latter by stating that $y_i$ is distributed according to some (possibly unknown) distribution with mean $f(x_i)$ and a strong tail-decay property. 
\begin{assumption}
\label{ass:observe}
There exists a constant $\sigma >0$ such that for each sampled $x_i$, the corresponding  observation $\observ_{i}$  is distributed according to a $\sigma^2$-sub-Gaussian distribution $M_{x_i}$ satisfying $\int u dM_{x_i}(u) = f(x_i)$.
\end{assumption}

In the next section, we present the HOO-MB algorithm.
In Section~\ref{sec:regret-mfhoo} we present its analysis leading to the regret bound. In Section~\ref{sec:smc-mfhoo} we discuss how HOO-MB can be used for solving the verification problem of Equations~(\ref{eq:SMCproblem}) and~(\ref{eq:param_problem}). In Section~\ref{sec:discuss}, we discuss the choice of the  {\em batch size\/} parameter $\batch$ and smoothness parameters as well as their implications.  

\subsection{Hierarchical tree optimization with mini-batches}
\label{sec:hierarchical-mb}
 HOO-MB (Algorithm~\ref{Al:HOO_batch}) selects the next sample $x_{j+1}$ by building a binary tree in which each height (or level) partitions  the state space $\X$ into a number of regions. The algorithm samples states to estimate  upper-bounds on $f(x)$ over a region, and based on this estimate, decides to expand certain branches (i.e., re-partition certain regions). The partitioning of nodes presents a tension between (a)  sampling the same state $x_i$ multiple times to reduce the variance in the estimate of $f(x_i)$ obtained from the noisy observations $y_i$, and (b) sampling different states to reduce the region sizes based on the smoothness of $f$. HOO-MB allows us to tune this choice using the batch size parameter $\batch$. In Section~\ref{sec:discuss} and \ref{sec:eval-mfhoo}, we will discuss the implications of this choice.
 
First, we discuss the tree data-structure. Each node in the  tree is labeled by a pair of integers $(h, i)$, where $h\geq 0$ is the height, and $i \in  \{1,\ldots, 2^h\}$, is its position within level $h$. The root is labeled $(0,1)$. Each node  $(h,i)$ can have two children $(h+1, 2i-1)$ and $(h+1, 2i)$. 
Node  $(h, i)$ is associated with the region $\partition{h}{i} \subseteq \X$, where $\partition{h}{i}=\partition{h+1}{2i-1}\cup \partition{h+1}{2i}$, and for each $h$ these disjoint regions satisfy $\cup_{i=1}^{2^h} \partition{h}{i} = \X$. Thus, larger values of $h$ represent finer partitions of $\X$.

For each node $(h,i)$ in the tree, HOO-MB computes  the following quantities:
\begin{inparaenum}[(i)]
\item $t_{h,i}$ is the number of times the node is chosen or considered for re-partitioning.
\item $\counter{h}{i}$ is the number of times the node is sampled. For batch size $\batch$, every time the node $(h,i)$ is chosen, it is sampled $\batch$ times.
\item $\hat{f}_{h,i}$ is the empirical mean of observations over points sampled in $\partition{h}{i}$. 
\item  $U_{h,i}$ is an initial estimate of the upper-bound of $f$ over $\partition{h}{i}$ based on the smoothness parameters. 
\item $B_{h,i}$ is a tighter and optimistic upper bound for the same.
\end{inparaenum}
The $\mathit{tree}$ starts  with a single root  $(0,1)$, with $B$-values of its two children $B_ {1,1}$ and $B_{1,2}$ initialized to $+\infty$. At each iteration a $\mathit{path}$ from the root to a leaf is found by traversing the child with the higher $B$-value (with ties broken arbitrarily), then a new node $(\mathit{hnew},\mathit{inew})$ is added and all of the above quantities are updated. The partitioning continues until the sampling budget $N$ is exhausted. Once the sampling budget $N$ is exhausted, a leaf with maximum $B$-value at the maximum depth $h_N$ is returned. The details are provided in Algorithm~\ref{Al:HOO_batch}. 
 
 
 \begin{algorithm}[ht!]
\caption{\small HOO-MB with parameters: sampling budget $N$, noise parameter $\sigma$, smoothness parameters $\nu>0$, $\rho\in(0,1)$, batch size $\batch$.}
\label{Al:HOO_batch}
\begin{algorithmic}[1]
\State $\mathit{tree} =\{(0,1)\}$, $B_{1,1}=B_{1,2}=\infty$
	\While{$n <= N$} $\label{ln:b_outerwhile}$
        \State $(\mathit{path}, (\mathit{hnew},\mathit{inew}))$  $\leftarrow \mathit{Travrse}(\mathit{tree})$ $\label{ln:b_path}$ 
        \State $\textbf{choose}$
        $x\in\mathcal{P}_{H,I}$ $\label{ln:b_choose}$
        \State \begin{varwidth}[t]{\linewidth} $\textbf{query}$ $x$ and get $\batch$ observations $y_1,y_2,...,y_b$ $\label{ln:b_recieve}$
        \end{varwidth}
    	\State $\mathit{tree.Insert}((\mathit{hnew},\mathit{inew}))$ $\label{ln:b_insert}$
    	\ForAll{$(h,i)\in \mathit{path}$} $\label{ln:b_update-path}$
    	    \State $t_{h,i} \leftarrow t_{h,i} + 1$
        	\State $count_{h,i} \leftarrow count_{h,i} + \batch$
        	\State \begin{varwidth}[t]{\linewidth}
            $\hat{f}_{h,i} \leftarrow (1- \frac{\batch}{\counter{h}{i}})\hat{f}_{h,i}$  \par
            \hspace{2cm}$+ \frac{\sum_{j=1}^{\batch} \observ _j}{\counter{h}{i}}$
            \end{varwidth}
    	\EndFor
    	\State \begin{varwidth}[t]{\linewidth} $m \leftarrow m + 1$
    	\par
    	$n \leftarrow n + \batch$
    	\par  
    	$B_{\mathit{hnew}+1,2\mathit{inew}-1} \leftarrow +\infty$,
    	\par
    	$B_{\mathit{hnew}+1,2\mathit{inew}} \leftarrow +\infty$
    	\end{varwidth}
        \ForAll{$(h,i)\in \mathit{tree}$} leaf up:              $\label{ln:b_update-tree}$
        	\State \begin{varwidth}[t]{\linewidth}
        	$U_{h,i} \leftarrow \hat{f}_{h,i}+\sqrt{\frac{2\sigma^2\ln m}{\batch t_{h,i}}}$ $+\nu\rho^h$
        	\end{varwidth}
            \State \begin{varwidth}[t]{\linewidth}
            $B_{h,i} \leftarrow \min\{U_{h,i},$ \par \hskip\algorithmicindent \hspace{0.2cm} $\max\{B_{h+1,2i+1}, B_{h+1,2i}\}\}$
            \end{varwidth}
    	\EndFor
	\EndWhile\\
	\Return $\underset{(h,i)\in \mathit{tree}}{\rm argmax}\ B_{h,i}$ at depth $h_N$
	\end{algorithmic}
\end{algorithm}

\subsection{Analysis of Regret Bound}
\label{sec:regret-mfhoo}

The analysis of the regret bounds for  HOO-MB follows the pattern of analysis in~\cite{bubeck2011x,sen2019noisy} and the details are given in the Appendix~\ref{appendix:proof-th} and \ref{appendix:proof-l}.

First, we define some notations: As in~\cite{bubeck2011x}, let $\Delta_{h,i}$ denote the {\em optimality gap\/} of node $(h,i)$, that is, $\Delta_{h,i} = f^* - \sup_{{x\in\mathcal{P}_{h,i}}} f(x)$. We say that a node $(h,i)$ is $\epsilon$-optimal if $\Delta_{h,i} \leq \epsilon$. A node $(h,i)$ is optimal if $\Delta_{h,i} = 0$ and it is sub-optimal if $\Delta_{h,i} > 0$. Let $(h,i^*)$ be the optimal node at depth $h$. 
We will use two  parameters $\nu$ and $\rho\in(0,1)$ to characterize the {\em smoothness} of $f$ relative to the partitions (see Assumption~\ref{Ass:HOO}). Roughly, these parameters restrict how quickly $f(x)$ can drop-off near the optimal $x^*$ within a $\partition{h}{i}$.

We define $I_h$ as in~\cite{grill2015black} as the set of all nodes at depth $h\geq 0$ that are $2\nu\rho^h$-optimal, and $\mathcal{N}_h(\epsilon)$ as the number of {\em $\epsilon$-optimal cells at depth $h$\/}, that is, the number of cells $\mathcal{P}_{h,i}$ with $\Delta_{h,i} \leq \epsilon$. That is,  $|I_h| = \mathcal{N}_h(2\nu\rho^h)$.

From the sampled estimate of $f(x_i)$ at a single point, HOO-MB attempts to estimate the maximum possible value that $f^*$ can take over $\mathcal{X}$. This is achieved by assuming that $f$ is locally smooth, i.e., there is no sudden large drop in the function around the global maximum. 
\begin{assumption}
\label{Ass:HOO}
There exist $\nu>0$ and $\rho \in (0, 1)$ such that for all $(h,i)$ satisfying $\Delta_{h,i} \leq c\nu\rho^h$ (for a constant $c \geq 0$), for all $x\in\mathcal{P}_{h,i}$ we have $f^* - f(x) \leq \max\{2c, c + 1\}\nu\rho^h$.
\end{assumption}

Here $c$ is a parameter that relates the variation of $f$ over $c\nu\rho^h$-optimal cells. For  $c=0$ it  implies that that there exist smoothness parameters such that the gap between the $f(x^*)$ and the value of $f$ over optimal cells is bound by $\nu\rho^h$; for $c=2$, it implies that there exist smoothness parameters such that over all $2\nu\rho^h$-optimal cells, the gap between the $f(x^*)$ and value of $f$ over those cells is bounded by $4\nu\rho^h$-optimal, and so on. For a sampling budget $N$ the final constructed {\em tree} has a maximum height $h_{max}$. For this $h_{\mathit{max}}$, this assumption allows function $f$ to have bounded jump around global maximum.

We now define a modified version of {\em near-optimality dimension} concept which plays an important role in the analysis of black-box optimization algorithms~\cite{Bubeck:2011,sen2019noisy,grill2015black,valko2013stochastic}. It measures the dimension of sets that are close to optimal.
\begin{definition}
\label{def:mod_near_opt}
$h_{\mathit{max}}$-bounded near-optimality dimension of $f$ with respect to  $(\nu, \rho)$ is:
$d_{m}(\nu, \rho)=\inf\{d'\in\mathbb{R}_{>0}:\exists B>0,\ \forall h \in [0,h_{max}],
\mathcal{N}_h(2\nu\rho^h) \leq B\rho^{-d'h}\}$.
\end{definition}
In other words, the $h_{\mathit{max}}$-bounded near-optimality dimension is the smallest $d_m \geq 0$ such that  the number of $2\nu\rho^h$-optimal cells at any depth $0\leq h \leq h_{max}$ is bounded by $B\rho^{-d_{m}h}$, for a constant $B$.
The number $\mathcal{N}_h(2\nu\rho^h)$ of near-optimal cells grow exponentially with $h$, and the near-optimality dimension 
gives the exponential rate of this growth. Thus,  $|I_h| = \mathcal{N}_h(2\nu\rho^h) 
\leq B\rho^{-d_{m}(\nu,\rho)h}$.



We are now ready to sketch regret bound for HOO-MB. 

\begin{theorem}
\label{Th:regret_HOO_batch}
With the input parameters satisfying Assumptions~\ref{ass:observe} and~\ref{Ass:HOO} and a sampling budget of $N$, HOO-MB achieves a regret bound of 
\begin{align*}
	\mathbb{E}[S_N] = O\bigg( \big(\dfrac{B \log \big( \lfloor \dfrac{N-1}{\batch} \rfloor + 1 \big)  +\batch}{N}\big)^{\frac{1}{d_{m}+2}}\bigg),
\end{align*}
where $d_{m}=d_{m}(\nu,\rho)$ is the near-optimality dimension and $B$ is the constant appearing in Definition~\ref{def:mod_near_opt}.
\end{theorem}

\begin{proof}(sketch)
Let $x_i \in \mathcal{X}$ be the point returned by HOO-MB at round $i$. %
Let $R_N = \sum_{i=1}^N (f^*-f(x_i))$ be the cumulative regret at round $N$.
Let $\mathcal{T}$ be the $\textit{tree}$ constructed by HOO-MB at the end of $N$ iterations. Let us fix an arbitrary height $H>0$ and based on $H$ we partition $\mathcal{T}$ into three sub-trees (See Figure~\ref{fig:tree}): 
$\mathcal{T}_1$ contains the nodes that are $2\nu\rho^h$-optimal at depth $h \geq H$; 
$\mathcal{T}_2$ contains the nodes that are $2\nu\rho^h$-optimal at depth $h < H$; and $\mathcal{T}_3$ has all other nodes.
For a node that is not $2\nu\rho^h$-optimal at depth $h$, one of its ancestors at some depth $h_a < h$ is $2\nu\rho^{h_a}$-optimal. The sub-tree $\mathcal{T}_3$ includes all such nodes, i.e. it includes the descendants of any node at depth $h$ that is not $2\nu\rho^h$-optimal but its parent is $2\nu\rho^{h-1}$-optimal. 

\begin{figure}
    \centering
    \includegraphics[width=1\linewidth]{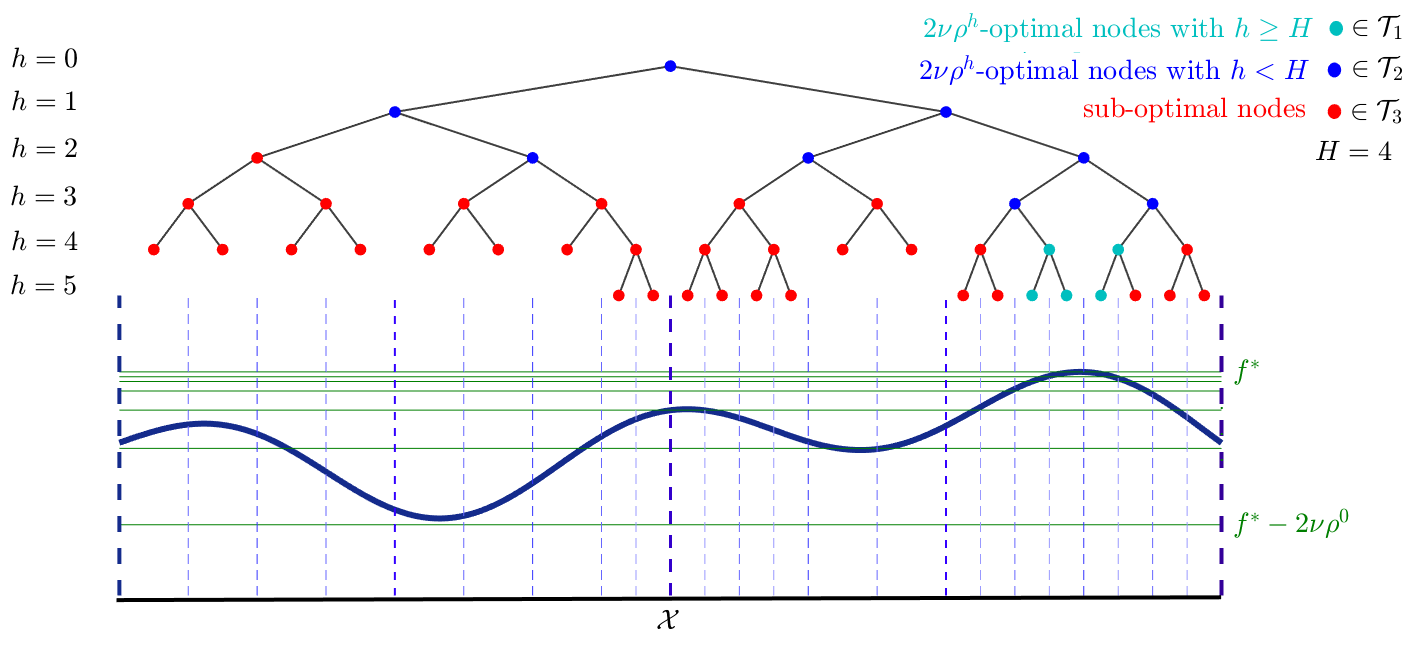}
    \caption{\small As in HOO \cite{munos:hal-2014,bubeck2011x}, the \textit{tree} $\mathcal{T}$ constructed by HOO-MB and its decomposition into sub-trees with the parameter $H=4$. Vertical (dashed) lines represent the $\partition{h}{i}$ boundaries constructed by $\mathcal{T}$. Cyan ($\mathcal{T}_1$) and blue ($\mathcal{T}_2$) nodes are $2\nu\rho^h$-optimal at $0 \leq h \leq 5$ as marked by the horizontal lines. 
    The red nodes are sub-optimal and belong to $\mathcal{T}_3$.}
    \label{fig:tree}
\end{figure}

Let $R_{N,1}$, $R_{N,2}$ and $R_{N,3}$ be the  cumulative regrets for the nodes that belong to the sub-trees $\mathcal{T}_1$, $\mathcal{T}_2$ and $\mathcal{T}_3$, respectively.
The sub-tree $\mathcal{T}_1$ contains the nodes $(h,i)$ that are at least $2\nu\rho^H$-optimal. By  Assumption~\ref{Ass:HOO}, we  conclude that all the points $x\in\partition{h}{i}$ satisfy $f^* - f(x) \leq 4\nu\rho^H$ and summing over $\mathcal{T}_1$ using $|\mathcal{T}_1|\leq N$, we get  $\mathbb{E}[R_{N,1}] \leq 4\nu\rho^HN$.
    
The sub-tree $\mathcal{T}_2$ contains the nodes $(h,i)$ that are $2\nu\rho^h$-optimal. Again, using  Assumption~\ref{Ass:HOO} we can conclude that all the points $x\in\partition{h}{i}$ satisfy $f^* - f(x) \leq 4\nu\rho^h$. In addition, since $\mathcal{T}_2$ is the union of sets $I_h$ for $h=0,\dots,H-1$, using the Definition~\ref{def:mod_near_opt}, the size of sub-tree $\mathcal{T}_2$ can be bounded. Combining these we can get $\mathbb{E}[R_{N,2}] \leq \batch \sum_{h=0}^{H-1} 4\nu\rho^{h}|I_h| \leq 4\batch \nu B \sum_{h=0}^{H-1} \rho^{h(1-d_{m})}$.

 The sub-tree $\mathcal{T}_3$ contains sub-optimal nodes and we can upper bound the expected number of visits to sub-optimal nodes using the modified form of Lemma $14$ from~\cite{bubeck2011x} 
 %
 and get $$\mathbb{E}[R_{N,3}] \leq 8\batch\nu B \sum_{h=1}^{H} \rho^{(h-1)(1-d_{m})} \big( \dfrac{8\sigma^2\log{(\lfloor \dfrac{N-1}{\batch} \rfloor+1)}}{\batch\nu^2\rho^{2h}} + 4 \big),$$ where the details can be found in Appendix~\ref{appendix:proof-l}. 
   
Combining the three upper bounds for $R_{N,1}$, $R_{N,2}$ and $R_{N,3}$ we get an expression of an upper bound on $\mathbb{E}[R_N]$ involving $H$. We minimize this  over $H$ to get $\mathbb{E}[R_N]$ as $O\bigg( BN^{\frac{d_{m}+1}{d_{m}+2}}\ \bigg(\log (\lfloor \dfrac{N-1}{\batch} \rfloor + 1 ) + \batch \bigg)^{\frac{1}{d_{m}+2}}\bigg)$.
Note that  if one of the points sampled by the algorithm were chosen uniformly at random as the final output, then,  $\mathbb{E}(S_N) \leq \dfrac{\mathbb{E}[R_N]}{N}$~\cite{bubeck2011x,bubeck2011pure}. Since, at round $N$, our algorithm returns the best sampled point, this relation  also applies for our algorithm. 
\hfill $\blacksquare$
\end{proof}

From Theorem~\ref{Th:regret_HOO_batch}, we observe that the regret is minimized if the $h_{max}$-bounded near-optimality dimension $d_{m}(\nu, \rho)$ is minimized. If the  smoothness parameters $(\nu,\rho)$ for the function $f$ that minimize the near-optimlaity dimension $d_{m}$ are known, then \toolname{} achieves this minimum regret. In general, we believe that inferring the minimizing smoothness parameters for a given verification problem will be challenging and requires further investigations. However, if bounds on these smoothness parameters are known---which is often the case for physical processes---then the search for the optimal parameters can be parallelized with similar regret bounds as given by Theorem~\ref{Th:regret_HOO_batch}~\cite{grill2015black}. 
The pseudocode for this parallel search is given in~\cite{grill2015black}.

\subsection{Verification and Parameter Tuning with HOO-MB}
\label{sec:smc-mfhoo}
\paragraph{Verification}
In order to use HOO-MB for verification, a natural choice for the objective function  would be to define $f(x):= \phit{k}{\Unsafe}{x}$ for any initial state $x \in \Theta$. Evaluating this function $\phit{k}{\Unsafe}{x}$ exactly is infeasible when the transition kernel $\mathbb{P}_{\beta}$ is unknown. Even if $\mathbb{P}_{\beta}$ is known, calculating $\phit{k}{\Unsafe}{x}$ involves integral over the state space (as in~(\ref{eq:prob_int})). 
Instead, we take advantage of the fact that HOO-MB can work with noisy observations. For any initial state $x \in \Theta$, and an execution $\alpha$ starting from $x$ we define the observation:
\begin{align}
\Observ = 1 \ \mbox{if} \ \alpha\ \mbox{hits}\ \Unsafe\ \mbox{within}\ k\ \mbox{steps}, \ \mbox{and}\ 0 \ \mbox{otherwise}.
\end{align}
Thus, given an initial state $x$, 
$\Observ=1$ with probability $\phit{k}{\Unsafe}{x}$, and $\Observ=0$ with probability $1-\phit{k}{\Unsafe}{x}$. That is, $\Observ$ is a Bernoulli random variable with mean $\phit{k}{\Unsafe}{x}$. In HOO-MB, once an initial state $x \in \partition{\mathit{hnew}}{\mathit{inew}}$ is chosen (line~\ref{ln:b_choose}), we simulate $\M$ upto $k$ steps $\batch$ times starting from $x$  and calculate the empirical mean of $Y$, which serves as the noisy observation $y$.

\paragraph{Parameter Synthesis}
In order to use HOO-MB for parameter synthesis, a natural choice for the objective function would be $f(\beta):= \mathbb{E}[r(\alpha,\beta)|\beta]$. Evaluating this function exactly is infeasible when the transition kernel $\mathbb{P}_{\beta}$ is unknown. 
Instead, we take advantage of the fact that HOO-MB can work with noisy observations. For a parameter $\beta \in \mathcal{B}$, and an execution $\alpha$ starting from $x_0$ sampled from a given distribution we receive an observation $\Observ$. The observation corresponding to this simulation is $\Observ=r(\alpha,\beta)$, with mean $\mathbb{E}[r(\alpha,\beta)|\beta]$. In HOO-MB, once an parameter $\beta \in \partition{\mathit{hnew}}{\mathit{inew}}$ is chosen (line~\ref{ln:b_choose}), we repeat the above simulation $\batch$ times and calculate the empirical mean of $\Observ$.\\

Assume each observation $\Observ$ satisfy conditions of  Assumption~\ref{ass:observe}. Then we have the following proposition. 
\begin{proposition}
\label{opt-gap}
Given smoothness parameters $\rho$ and $\nu$ satisfying Assumption~\ref{Ass:HOO} for the function $f(x) := p_{k,\Unsafe}(x)$ or $f(\beta) := \mathbb{E}[r(\alpha, \beta)|\beta]$, 
if Algorithm~\ref{Al:HOO_batch} returns  $\bar{x}_N \in \Theta$ or $\mathcal{B}$, then
$\phit{k}{\Unsafe}{x^*} - \phit{k}{\Unsafe}{\bar{x}_N}$ or $\mathbb{E}[r(\alpha,\beta)|\beta^*]-\mathbb{E}[r(\alpha,\beta)| \bar{x}_N]$, is upper bounded by Theorem~\ref{Th:regret_HOO_batch}.
\end{proposition}

\subsection{Discussions on choice of  parameter values}
\label{sec:discuss}
\textbf{Batch size $\batch$:}
The main difference between HOO-MB and the original HOO~\cite{bubeck2011x} is that each node $(h,i)$ in the HOO-MB is sampled $\batch$ times. In other words once a node (arm) is chosen, instead of a single observation, $\batch \geq 1$ observations are received. This in turn, required the update rules for $m, U_{h,i},$ and $B_{h,i}$ to be generalized. Indeed, by setting $\batch=1$ in HOO-MB we recover the original HOO algorithm and the corresponding simple regret bound  $\mathbb{E}[S_N] = O\bigg( (\frac{B\log N}{N})^{\frac{1}{d+2}}\bigg)$.
Comparing the regret bound in HOO-MB and HOO we observe that HOO-MB gets worse in terms of regret bound, however,  the number of nodes in the tree is reduced by a factor of $\batch$. This reduces the running time and makes HOO-MB more efficient in terms of memory usage with respect to the HOO algorithm. Experiments in Section~\ref{sec:performance} show that the actual regret of the algorithm doesn't increase much when using some reasonable batch sizes (e.g. $\batch=100$).\footnote{In fact, we observed that the actual regret first decreases and then increases as the batch size increases starting from $1$. This phenomena does not contradict the theoretical regret bound we have derived, since it is only an upper bound and may not be tight in some cases. More sophisticated theory has to be built to understand this phenomena, which we leave for future work.}

\textbf{Optimal smoothness parameters $(\nu, \rho)$:}
Recall from Section~\ref{sec:prelims}, $\phit{k}{\Unsafe}{x}$ is continuous in $x \in \Theta$, if the transition kernel $\mathbb{P}$ is continuous in $x \in \Theta$. This gives us a sufficient condition for guaranteeing the existence of  smoothness parameters $\nu$ and $\rho$  satisfying  Assumption~\ref{Ass:HOO}.
However, as mentioned in Section~\ref{sec:regret-mfhoo}, the optimal smoothness parameters $(\nu,\rho)$ minimizing the near-optimality dimension $d(\nu,\rho)$ of $\phit{k}{\Unsafe}{\cdot}$ are generally not known. Finding bounds on the optimal values of these parameters based on partial knowledge of the \modelname model would be a direction for future investigations. 
The following simple algorithm adaptively searches for the optimal smoothness parameters by spawning several parallel HOO-MB instances with various $\nu$ and $\rho$ values. This is a standard approach for parameter search in the bandits literature~(see for example~ \cite{sen2019noisy,grill2015black}).
\begin{algorithm}[H]
\caption{\small Parallel search with parameters: sampling budget $N$, number of instances $K$, and maximum smoothness parameters $\nu_{max}$ and $\rho_{max}$}\label{Al:MFPOO}
\begin{algorithmic}[1]
	\For{$i=1 : K$}
    	\State Spawn HOO-MB with $(\nu = \nu_{max}, \rho = \rho_{max}^{K/(K-i+1)})$ with budget $N/K$
	\EndFor
	\State Let $\bar{x}_{i}$ be the point returned by the $i^{th}$ HOO-MB instance for $i \in \{1, .., K\}$\\
	
	\Return $\{\bar{x}_{i} | \ i = 1, .., K\}$
\end{algorithmic}
\end{algorithm}

\section{\toolname{} tool and experimental evaluation}
\label{sec:experiments}

The components of  \toolname{}  are shown in Figure~\ref{fig:HooVer_Components}. Given the the initial state and/or the parameter, \toolname{} generates random trajectories  using the transition kernel simulator and gets rewards. It runs $K$ instances of HOO-MB with automatically calculated smoothness parameters. Each instance returns an estimate $\bar{x}_i$ of the optimum, then \toolname{} computes the mean of reward for each $\bar{x}_i$ using Monte-Carlo simulations\footnote{Number of simulations used in this step is included in ``\#queries" in Fig.\ref{fig:results}.}, and outputs the $\bar{x}$ that gives the highest mean reward. Source files and instructions for reproducing the results are available from the \toolname{} web page\footnote{\url{https://www.daweisun.me/hoover}. The source code is available at \url{https://github.com/sundw2014/HooVer}}.
\begin{figure}
    \centering
    \includegraphics[width=1\linewidth]{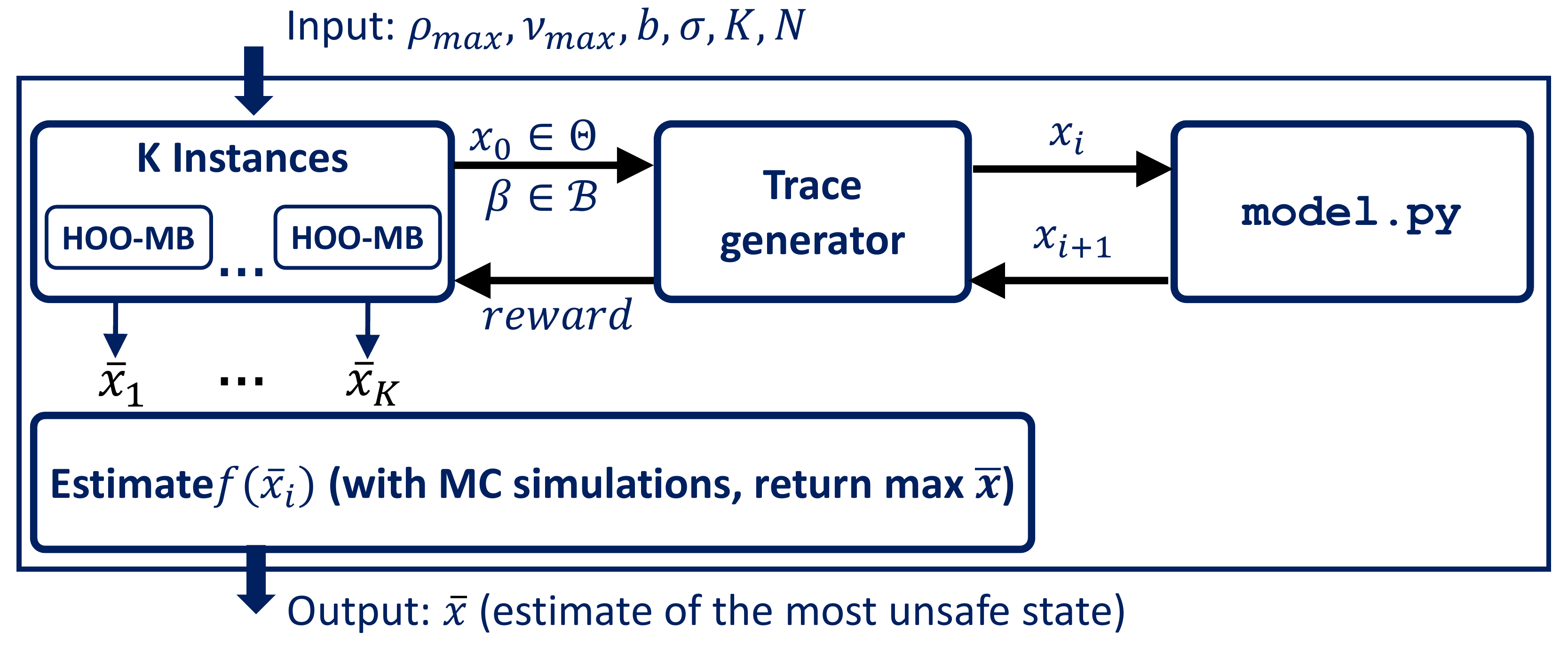}
    \caption{\small  \toolname{} tool. The parameters: $\rho_{max},\nu_{max}$ are used to calculate smoothness parameters. We fix $\nu_{max} = 1.0$ and results are not sensitive to  $\rho_{max}$ (Section~\ref{sec:performance}). The impact of batch size $\batch$ is discussed in Section~\ref{sec:performance}.
    The noise parameter $\sigma = 0.5$ is a valid choice for all models. The number of HOO-MB instances $K$ is fixed to $4$.
    }
    \label{fig:HooVer_Components}
\end{figure}

\subsection{Benchmarks}
\label{sec:allbenchmarks}

\begin{figure}
\centering
\includegraphics[width=1\linewidth]{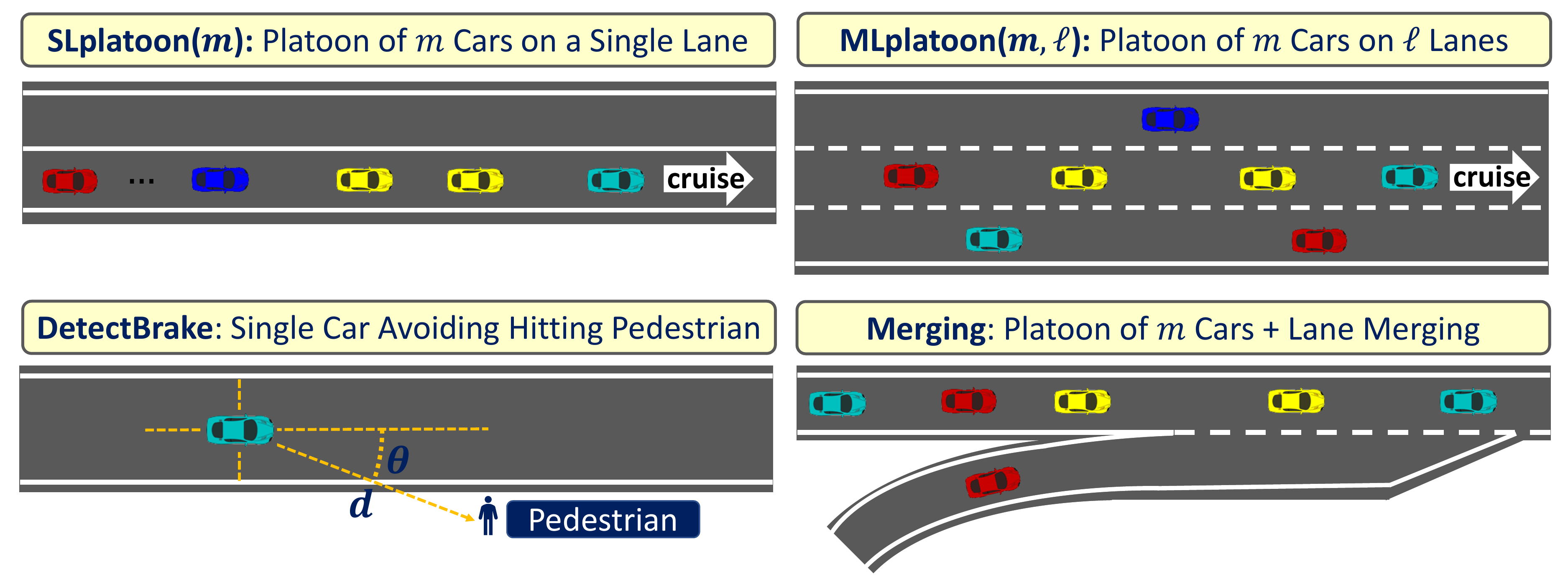}\par
\caption{\small Illustration of the verification benchmark scenarios which can be instantiated with different number of vehicles and initial conditions.}
\label{fig:scenarios}
\end{figure}

In order to evaluate the proposed method on verification tasks, we have created several example models (Figure~\ref{fig:scenarios}) that capture scenarios involving autonomous vehicles and driving assist systems. More than 25\% of all highway accidents are rear-end crashes~\cite{kodaka2003rear}. Automatic braking and collision warning systems are believed to improve safety, however, their testing and certification remain challenging~\cite{simonepresent} (see also \cite{FQM18} for discussion). Our  examples  capture typical scenarios used for certification of these control systems. 

{\sf SLplatoon\/} models $m$ cars on a single lane where each car  probabilistically decides to {\em speed up\/}, {\em cruise,\/} or {\em brake\/} based on its current gap with the predecessor.  \Mlplatoon\ is a similar model with $m \times \ell$ cars on $\ell$ lanes where cars can also choose to change lanes. \Merging\ models a car on the ramp trying to merge to $m$ cars on the left lane. \DetectingPedestrian\ models a pedestrian crossing the street in front of an approaching car which brakes only if its sensor detects it. 
In all these models, the initial state uncertainties ($\Theta$) are defined based on the localization error  (e.g., GPS error in position and velocity) and the unsafe sets ($\Unsafe$)   capture collisions. In our experiments, we use up to $18$ vehicles in $2$ lanes (in \Mlplatoon{}). These give \modelname instances with up to $18$-dimensional state spaces with an initial uncertainty $(\Theta)$ that spans $8$ continuous dimensions.
More detailed description of the  scenarios are given at the   \toolname{} webpage. 

Moreover, we evaluate the performance of \toolname{} on parameter synthesis tasks using an Linear–quadratic regulator (LQR) benchmark. Specifically, we consider the system of as $x_{t+1} = Ax_t + Bu_t + w_t$, where $w_t \in \mathcal{N}(0, \sigma^2 I)$ are independent and identically distributed (i.i.d) Gaussian noise. We want to search for a state feedback gain matrix $K$ such that $u_t = Kx_t$ minimizes the cost $J(K) = \mathbb{E}_w\left[\sum_{t=0}^{T-1} (x_t^\intercal Q x_t + u_t^\intercal R u_t) + x_T^\intercal Q x\right]$. In the experiments, we consider the case where $x_t \in \mathbb{R}^2$ and $u_t \in \mathbb{R}^2$, and $\sigma$ is set to $0.01$. The distribution for the initial state is a Dirac delta function, i.e. the initial state is fixed.

\subsection{\toolname{} performance on benchmarks}\label{sec:performance}

\begin{figure}[t]
\begin{multicols}{2}
    \includegraphics[width=1\linewidth]{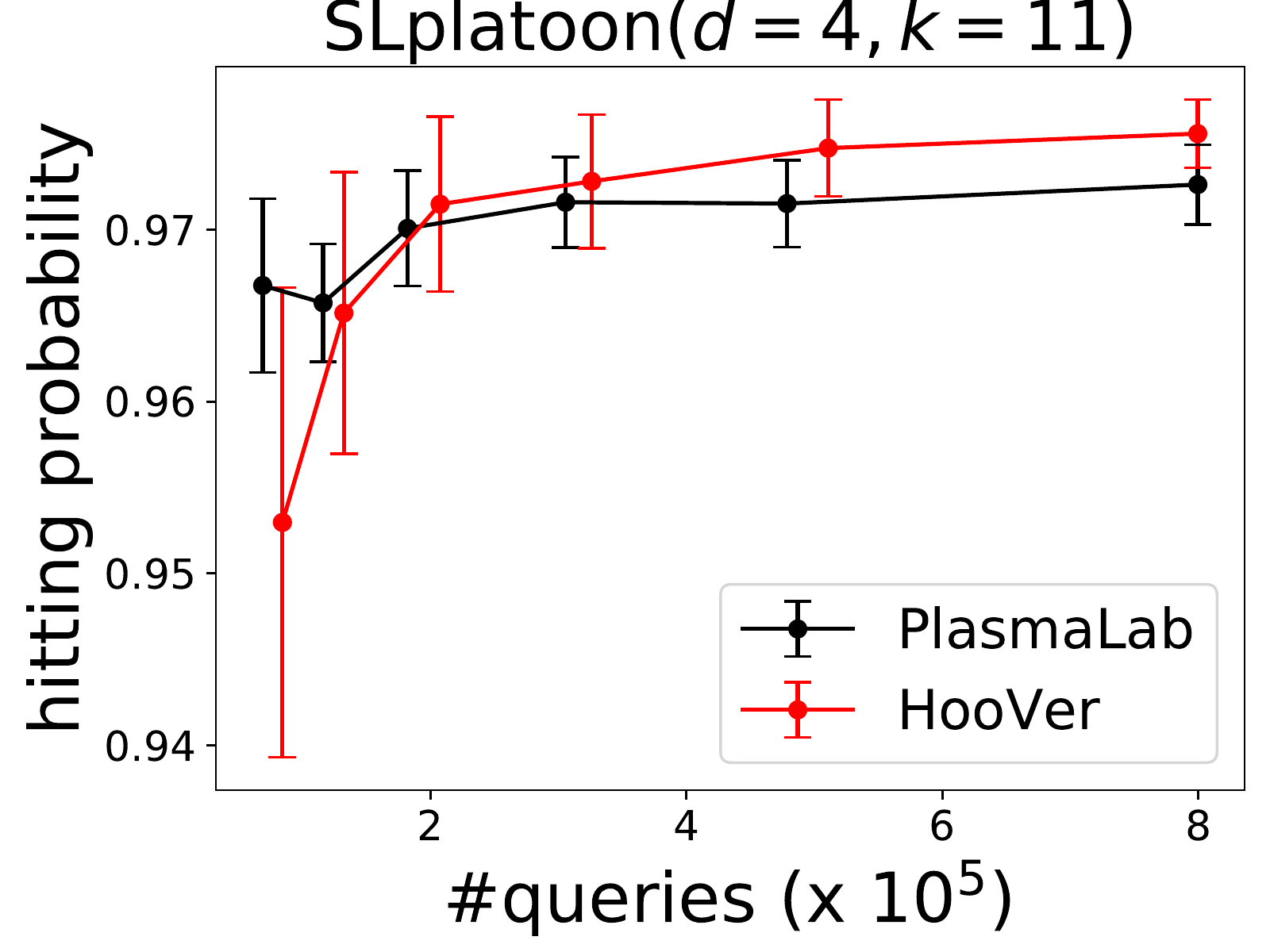}\par 
    \includegraphics[width=1\linewidth]{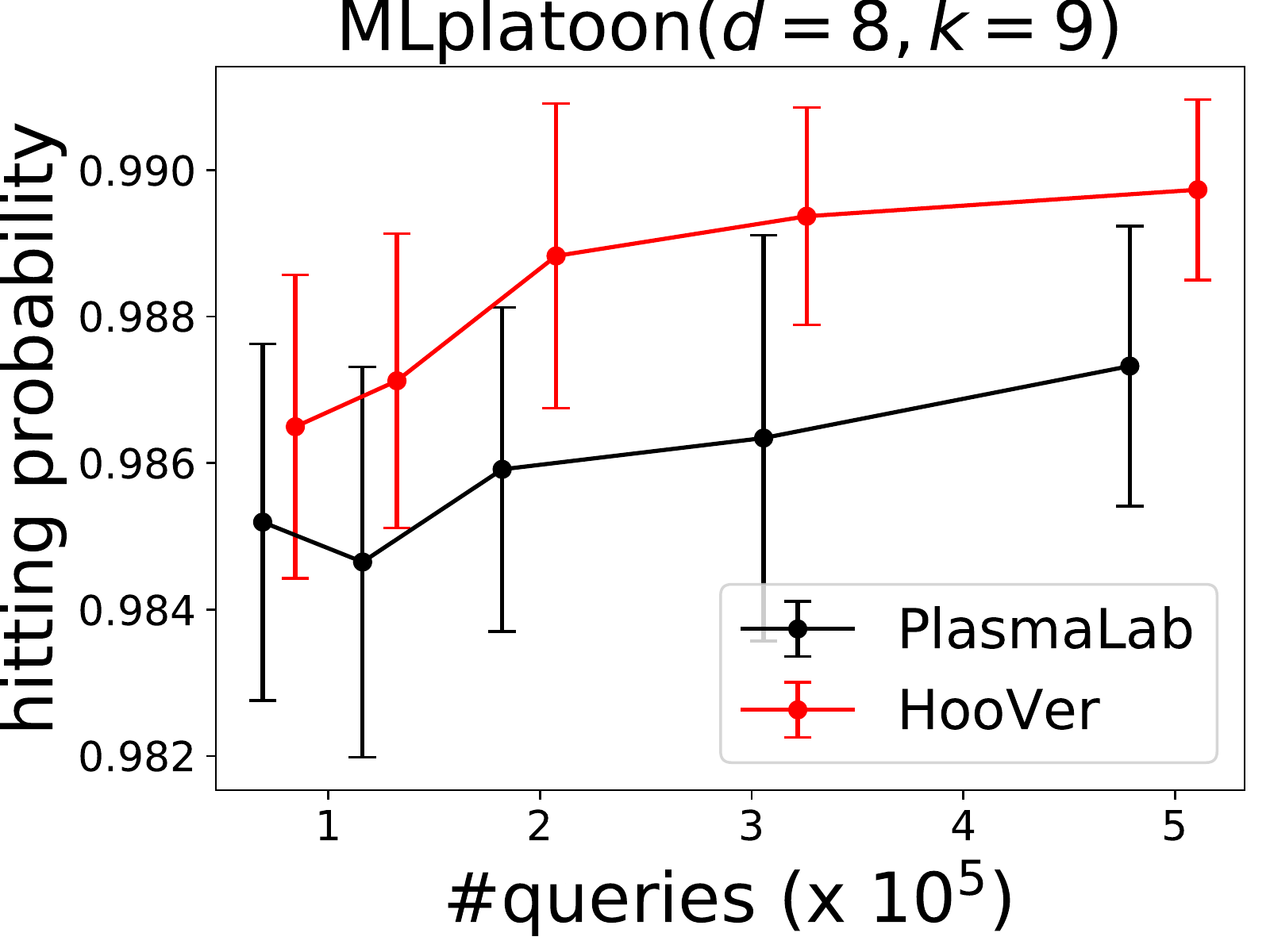}\par 
\end{multicols}
\begin{multicols}{2}
    \includegraphics[width=1\linewidth]{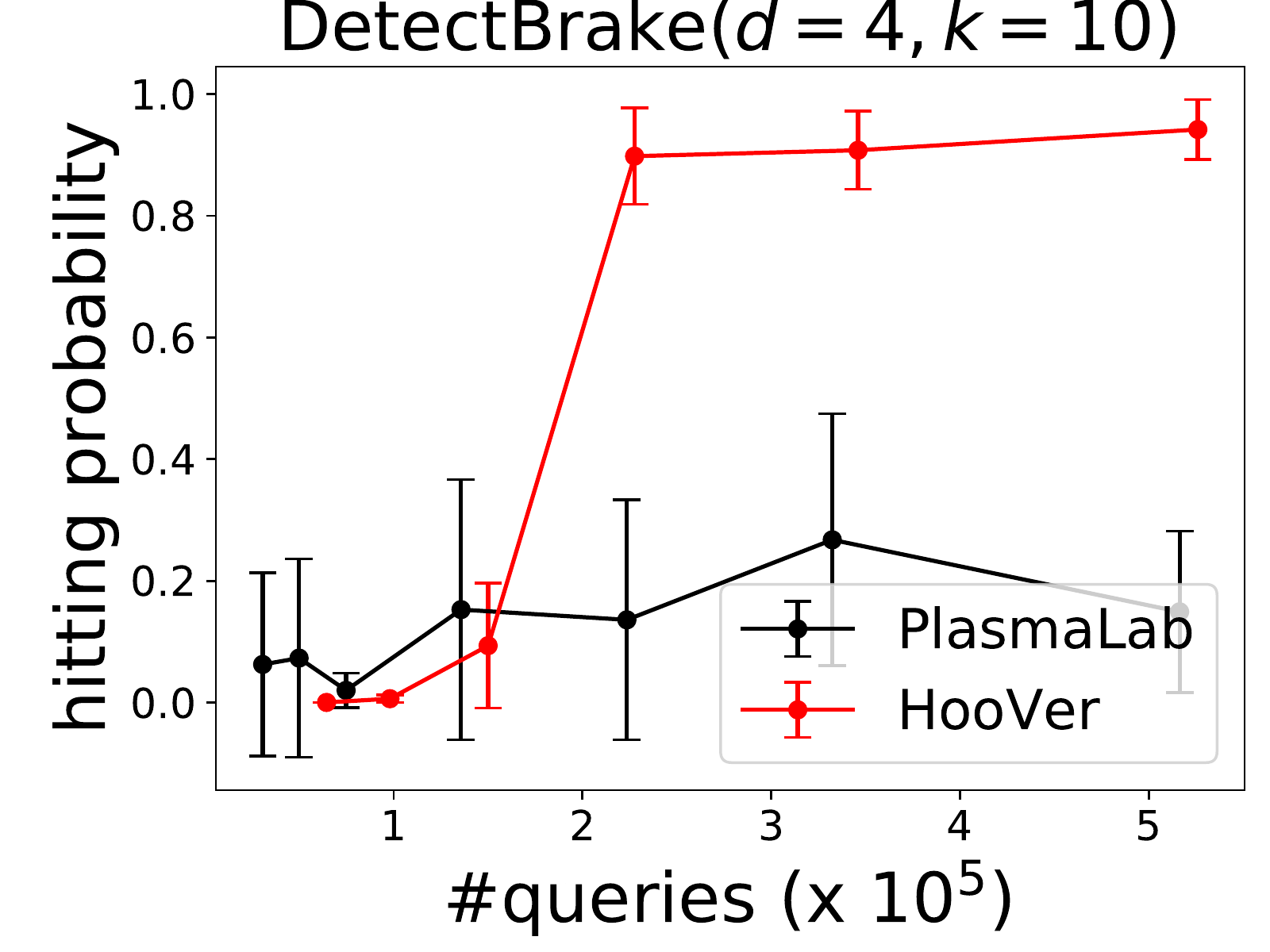}\par
    \includegraphics[width=1\linewidth]{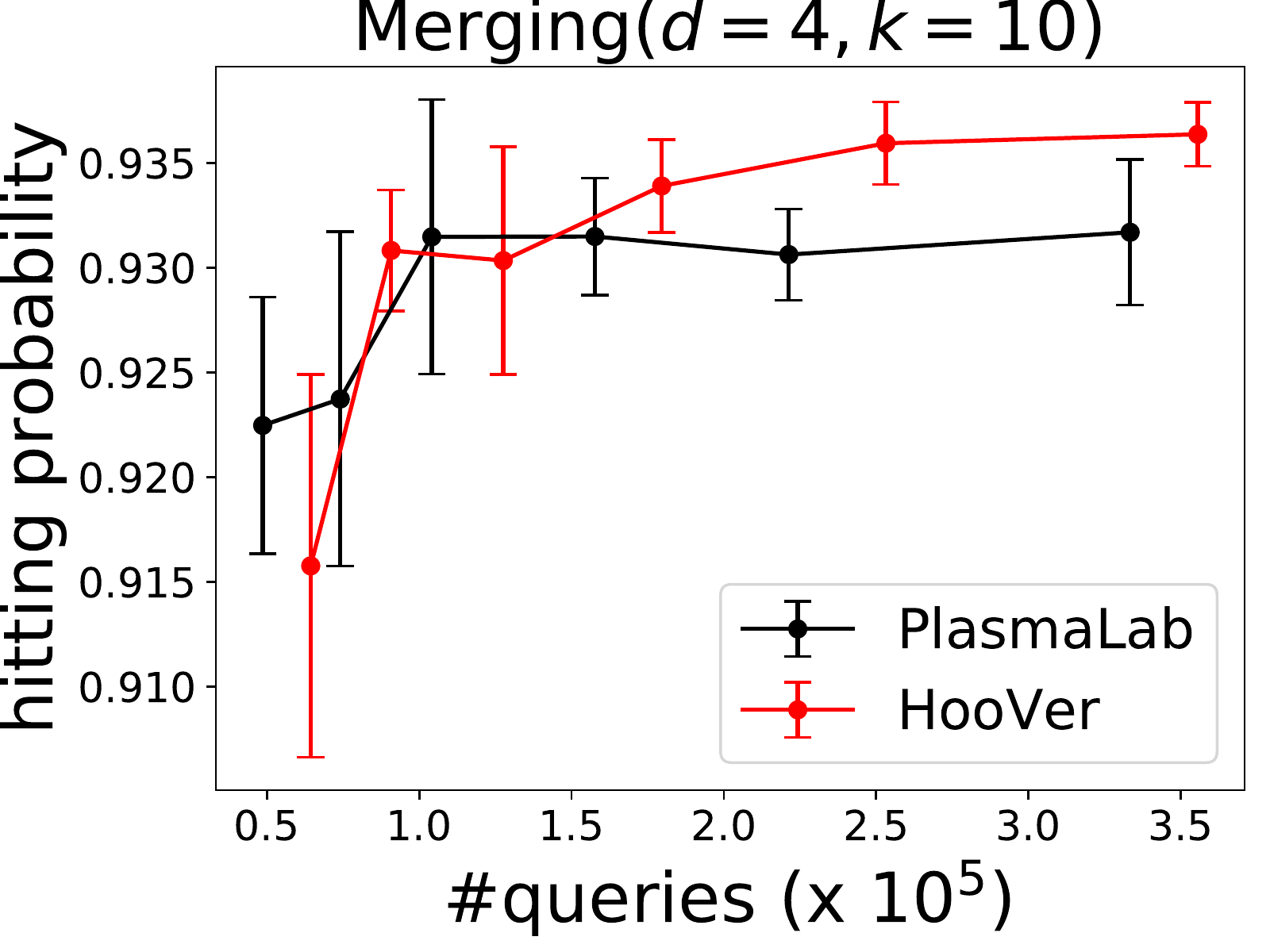}\par
\end{multicols}
\caption{\small Results on the verification benchmarks, where $d$ denotes the number of dimensions of $\Theta$, and $k$ denotes the time bound in $\phit{k}{\Unsafe}{\bar{x}_N}$. Mean and standard deviations are averaged over 10 runs. Hitting probability $\phit{k}{\Unsafe}{\bar{x}_N}$ is estimated using Monte Carlo method once $\bar{x}_N$ is returned from the tool.}
\label{fig:results}
\end{figure}

All our experiments were conducted on a Linux workstation with two Xeon Silver 4110 CPUs and 32 GB RAM.
Figure~\ref{fig:results} shows the performance of \toolname{} on above verification benchmarks. For each benchmark, the plots show the evolution of the estimated worst safety violation probability $\phit{k}{\Unsafe}{\bar{x}_N}$ against sampling (query) budget ($N$) for a fixed time horizon $k$. Since the actual maximum $\sup\limits_{x \in \Theta} \phit{k}{\Unsafe}{x}$ is not known, we can not evaluate the simple regret $\sup\limits_{x \in \Theta} \phit{k}{\Unsafe}{x} - \phit{k}{\Unsafe}{\bar{x}_N}$. However, comparing $\phit{k}{\Unsafe}{\bar{x}_N}$ directly is enough as the maximum is constant.

We make  two main observations: 
First, as expected, the output $\phit{k}{\Unsafe}{\bar{x}_N}$ improves with the budget. 
Second, although our implementation is not remotely optimized for running time, the running times are reasonable. For example, for {\sf SLplatoon\/}, the running time is less than $1$ millisecond per simulation, i.e.,  about $10$ minutes for 800K queries.

Table~\ref{tab:LQR} shows the performance of \toolname{} on the parameter synthesis benchmark. The estimation error is calculated as $\|\hat{K} - K^*\|_F$, where $\hat{K}$ is the estimated optimal parameter, and $K^*$ is the ground truth obtained by solving the LQR problem with $A$, $B$, $Q$, and $R$ given explicitly. In contrast, \toolname{} treats the system as a black-box and only relies on the observations obtained from the simulations. As shown in the table, the parameter found by \toolname{} tends toward the ground truth as budget increases, and it achieves a small enough error using a reasonable number of simulations.

\begin{table}[!htpb]
\centering
\def\arraystretch{1.3}
\setlength\tabcolsep{3pt}
\caption{\small {Performance of \toolname{} on parameter synthesis.}}
\begin{tabular}{|c|c|c|c|c|c|c|c|}
\hline
{\bf \#queries} & 1K & 2K & 4K & 8K & 16K & 32K\\
\hline
{\bf $\|\hat{K} - K^*\|_F$} & 0.594 & 0.581 & 0.222 & 0.161 & 0.052 & 0.022\\
\hline
\end{tabular}
\label{tab:LQR}
\end{table}

Next we study the impact of two essential parameters used by the tool.
\paragraph{Impact of batch size.}
As discussed earlier in Section~\ref{sec:discuss}, the batch size parameter $\batch$ of HOO-MB can help improve the running time and memory usage without significantly sacrificing the quality of the final answer.
Table~\ref{tab:bs} shows this for  {\sf SLplatoon\/} with sampling budget $N = 800$K. The \#Nodes refers to the number of the nodes in the final tree generated by \toolname{}.
Notice that the final answer $\phit{k}{\Unsafe}{}$ does not change much when using reasonable batch sizes ($\batch \leq 400$), while the number of nodes in the tree is reduced, which in turn can reduce the running time and the memory usage by orders of magnitude. Using a very large batch sizes (e.g. $\batch \geq 1600$) starts to affect the quality of the result, which is consistent with Theorem~\ref{Th:regret_HOO_batch}.

\begin{table}[!htpb]
\small 
\centering
\def\arraystretch{1.3}
\setlength\tabcolsep{3pt}
\caption{\small Impact of batch size $\batch$ on size of tree, final result, running time, and memory usage, for sampling budget of $800$K on {\sf SLplatoon\/} model. Results are averaged over 10 runs.}
\begin{tabular}{|c|c|c|c|c|c|c|c|c|}

\hline
{\bf $\batch$} & 10 & 100 & 400 & 1600 & 6400\\
\hline
{\bf \#Nodes} & 75996 & 7596 & 1900 & 468 & 116\\
\hline
{\bf Running Time\/} (s) & 2568 & 506 & 512 & 573 & 678\\
\hline
{\bf Memory (Mb)\/} & 67.16 & 6.62 & 1.64 & 0.38 & 0.09\\
\hline
{\bf Result\/} $\phit{k}{\Unsafe}{}$  & 0.9745 & 0.9756 & 0.9741 & 0.9667 & 0.8942\\
\hline
\end{tabular}
\label{tab:bs}
\end{table}



\paragraph{Impact of smoothness parameter.} 
Table~\ref{tab:rhomax} shows how smoothness parameter $\rho_{max}$ impacts the performance of \toolname. For large  values of $\rho_{max}$ (0.95 and 0.9), the upper confidence bounds (UCB) computed in HOO-MB forces the state space exploration to be more aggressive, and the algorithm explores shallower levels of the tree more extensively. As $\rho_{max}$ decreases, \toolname{} proceeds to deeper levels of the  tree. Below a threshold ($0.8$ in this case), the algorithm becomes insensitive to variation of $\rho_{max}$. Thus, if the smoothness of the model is unknown, one can select a small $\rho_{max}$, and obtain a reasonably good estimate for  result.

\begin{table}[!htpb]
\centering
\def\arraystretch{1.3}
\setlength\tabcolsep{3pt}
\caption{\small Impact of smoothness parameter $\rho_{max}$ on tree and final result. Results are averaged over 10 runs.}
\begin{tabular}{|c|c|c|c|c|c|c|c|}
\hline
${\bf \rho_{max}}$ & 0.95 & 0.90 & 0.80 & 0.60 & 0.40 & 0.16 & 0.01\\
\hline
{\bf Tree depth} & 11.6 & 14.3 & 25.3 & 25.4 & 25.6 & 24.2 & 24.3\\
\hline
{\bf Result\/} $\phit{k}{\Unsafe}{}$ & 0.9644 & 0.9647 & 0.9740 & 0.9756 & 0.9754 & 0.9728 & 0.9722\\
\hline
\end{tabular}
\label{tab:rhomax}
\end{table}

\subsection{Comparison with PlasmaLab}
\label{sec:eval-mfhoo}
Model checking tools such as Storm~\cite{dehnert2017storm} and PRISM~\cite{HKNP06} do not support MDPs defined on continuous state spaces. 
It is possible to compare \toolname{} with these tools on discrete versions of these examples, but the comparison would  not be fair as the guarantees given these tools are different.
%
The SMC approach for stochastic hybrid systems presented in~\cite{ellen2015statistical} is closely related, but we could not find an implementation to compare against. Furthermore, \toolname{} uses the HOO-MB which  does not rely on a semi-metric on the state space as required by the algorithm used in~\cite{ellen2015statistical}. 
Among the tools that are currently available we found  PlasmaLab~\cite{legay2016plasma} to be  closest in several ways, and therefore, we decided to perform a deeper comparison with it. 

PlasmaLab uses a smart sampling algorithm~\cite{d2015smart} to assign the simulation budget efficiently to each scheduler of an MDP.
In order to use this algorithm, one has to set parameters $\epsilon$ and $\delta$ in the Chernoff bound, satisfying $N_{max} > \ln{(2/\delta)}/(2\epsilon^2)$, where $N_{max}$ is per-iteration simulation budget. We set the confidence parameter $\delta$  to $0.01$, and  given an $N_{max}$, the precision parameter $\epsilon$ is then obtained by $\epsilon = \sqrt{\ln{(2/\delta)}/(2\times 0.8\times N_{max})}$. In order to make a fair comparison, we developed a PlasmaLab plugin which enables PlasmaLab to use exactly the same external Python simulator as \toolname{}.

In our experiments,  \toolname{} gets better than PlasmaLab, 
as the sampling budget increases (see Figure~\ref{fig:results}). 
It is not surprising that for small budgets, before a threshold depth in the tree is reached, \toolname{} cannot give an accurate answer.  Once the number of queries exceeds the threshold, the tree-based sampling of \toolname{} works more efficiently than PlasmaLab for the given examples.

\paragraph{A conceptual example.}
To illustrate the above behavior of the tools, we consider a conceptual example with hitting probability given by $\phit{k}{\Unsafe}{x}$ directly without specifying $\mathbb{P}$, $k$ and $\Unsafe$. Given an initial state $x = (x_1,x_2)$,  $\phit{k}{\Unsafe}{x_1, x_2} = p_{max} \cdot \exp{(-\frac{(x_1-0.5)^2+(x_2-0.5)^2}{s})}$, where the parameter $s$  controls the slope of $\phit{k}{\Unsafe}{\cdot}$ around the maximum $\phit{k}{\Unsafe}{\frac{1}{2},\frac{1}{2}} = p_{max}$. The smaller the value of $s$, the sharper the slope.

\begin{figure}
\centering
\includegraphics[width=0.6\linewidth]{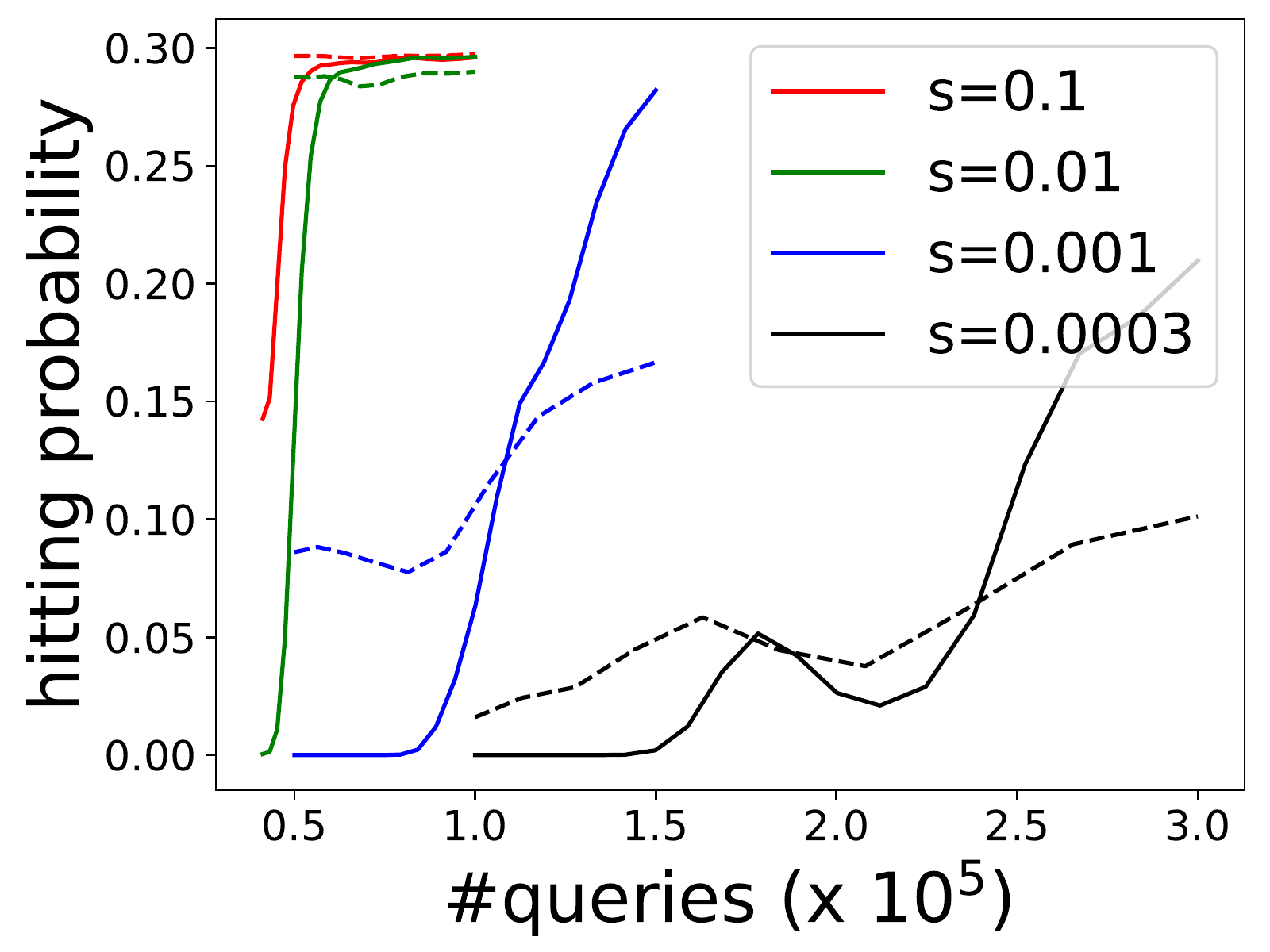}\par
\caption{\small \toolname{} (solid lines) and PlasmaLab (dashed) output in optimizing sharp functions. }
\label{fig:toy_example}
\end{figure}

In the experiments, we set $\Theta = \{(x_1,x_2)|0<x_1<1,~0<x_2<1\}$ and $p_{max} = 0.3$. Results are shown in Fig.~\ref{fig:toy_example}. With small query budget, PlasmaLab beats \toolname{}, however, as the budget increases \toolname{} improves swiftly and becomes better (see Figure~\ref{fig:toy_example}).
For ``easy'' objective functions (where $f_s$ is smooth, $s=0.1$), both tools perform well.
As $s$ decreases, $f_s$ becomes sharper and the most unsafe state become harder to find, the point where \toolname{} beats PlasmaLab moves to the  right. For sharp objective functions (e.g., $s=0.0003$), \toolname{} is much more sample efficient. 
This  suggests that \toolname{} might perform better than PlasmaLab in  SMC problems with  hard-to-find bugs or unsafe conditions.

\section{Conclusions}
\label{sec:conc}
We presented a new tree-based algorithm, HOO-MB, for verification and parameter synthesis of class of discrete-time nondeterministic, continuous state, Markov chains (MC) by building a connection with the multi-armed bandits literature. In this class of problem the uncertainty is in the initial states or parameters 
HOO-MB sequentially samples executions of the MC in batches and relies on a lightweight assumption about the smoothness of the objective function, to find near-optimal solutions violating the given safety requirement.
We provide theoretical regret bounds on the optimality gap in terms of the sampling budget, smoothness, near-optimality dimension, and sampling batch size. 
 We created several benchmarks models, implemented a tool (\toolname{}), and the experiments show that our approach is competitive compared to PlasmaLab in terms of sample efficiency.
%
 %
%
Detailed comparison with other verification and synthesis tools and exploration of general Markov decision processes with this approach would be directions for further investigation. 
\bibliographystyle{unsrt}
\bibliography{references,sayan1}
\appendix

\section{Appendix: Proof of Theorem \ref{Th:regret_HOO_batch}}
\label{appendix:proof-th}

We follow the proof of regret bound from~\cite{bubeck2011x,sen2019noisy}. Let $x_i \in \mathcal{X}$ be the point returned by Algorithm~\ref{Al:HOO_batch} at round $i$. Define cumulative regret at round $N$ as $$R_N = \sum_{i=1}^N (f^*-f(x_i)).$$ We first provide an upper bound for the cumulative regret and then obtain an upper bound for the simple regret.

In Algorithm~\ref{Al:HOO_batch}, once a node is chosen, it is queried for $\batch$ times. Let $L(N) = \left\lfloor \dfrac{N-1}{\batch} \right\rfloor$. Therefore, at round $N$, there are $L(N)+1$ mini-batches. Denote by $(h_j, i_j)$ the node chosen by HOO-MB in the $j$-th mini-batch, where $j = 1,\dots, L(N) +1$. Let $z_j$ denotes the arm played in the $j$-th mini-batch, where $z_1 = x_1 = \dots = x_{\batch},\ z_2= x_{\batch+1} = \dots = x_{2\batch}, \dots$. And $y_{j,k}$ denotes the $k$-th observation in the $j$-th mini-batch.

Using these notations, we rewrite the cumulative regret at round $N$ by splitting it into two parts, regret of the first $L$ batches and that of the last batch: $$\Expectation{R_N} = \sum_{j=1}^{L(N)} \batch (f^*-f(z_j)) + (N - \batch L(N)) \cdot (f^*-f(z_{L(N)+1})).$$

Now, let $I_h$ be the set of all nodes at depth $h$ that are $2\nu\rho^h$-optimal, where $h=0,1,2,...$. Hence, $I_0={(0,1)}$. Now let $I$ be the union of $I_h$'s. In addition, let $\mathcal{J}_h$ be the set of nodes at depth $h$ that are not in $I$ and whose parents are in $I_{h-1}$. Then denote by $\mathcal{J}$ the union of all $\mathcal{J}_h$'s.

Similar to~\cite{bubeck2011x}, we partition the tree $\mathcal{T}$ into three subsets, $\mathcal{T}=\mathcal{T}^1\cup\mathcal{T}^2\cup\mathcal{T}^3$, as follows. Let $H$ be an integer to be chosen later. The set $\mathcal{T}^1$ contains the descendants of the nodes in $I_H$ (by convention, a node is considered its own descendant); $\mathcal{T}^2=\cup_{0 \leq h <H}\ I_h$; $\mathcal{T}^3$ contains the descendants of the nodes in $\cup_{1 \leq h \leq H}\ \mathcal{J}_h$.

It is easy to see from the algorithm that any node chosen by the algorithm is chosen at most once. Since $(h_j, i_j)$ belongs to either of the $\mathcal{T}^i$'s, where $i=1,2,3$, we can decompose the cumulative regret according to which of the sets $\mathcal{T}^i$ the node $(h_j, i_j)$ belongs to:

\begin{equation}
    \Expectation{R_N}=\Expectation{R_{N,1} + R_{N,2} + R_{N,3}},
\end{equation}
where $R_{N,i}$ is the part of $R_N$ for all $(h_{j}, i_{j}) \in \mathcal{T}^i$. Specifically, we have

\begin{align*}
    & \Expectation{R_{n,i}} = \mathbb{E}\Biggl[\sum_{j=1}^{L(N)} \batch (f^*-f(z_j))\mathbb{I}_{\{(h_j, i_j) \in \mathcal{T}^i\}} \\
    & + (N-\batch L(N)) \cdot (f^*-f(z_{L(N)+1}))\mathbb{I}_{\{(h_{L(N)+1}, i_{L(N)+1}) \in \mathcal{T}^i\}}\Biggr]. 
\end{align*}

First, $\Expectation{R_{N,1}}$ is easy to bound. All nodes in $\mathcal{T}^1$ are $2\nu\rho^H$-optimal. According to Assumption~\ref{Ass:HOO}, all points in these cells are $4\nu\rho^H$-optimal. Thus we obtain the following
\begin{align*}
    \Expectation{R_{N,1}} \leq 4\nu\rho^HbL(N) + 4\nu\rho^H(N - \batch L(N)) = 4\nu\rho^HN.
\end{align*}

For $h\geq 0$, any node $(h,i)\in \mathcal{T}^2$ belongs to $I_h$, and thus is $2\nu\rho^h$-optimal. Therefore, all points in any cell $(h,i)\in \mathcal{T}^2$ are at least $4\nu\rho^h$-optimal. Also, by Definition~\ref{def:mod_near_opt}, we have that $|I_h| \leq B(\nu, \rho)\rho^{-d_m(\nu,\rho)h}$. Therefore, using the fact that each node in $\mathcal{T}^2$ is played at most in one mini-batch, we have the following:
\begin{align*}
    \Expectation{R_{N,2}} \leq b\sum_{h=0}^{H-1} 4\nu\rho^{h}|I_h| \leq 4b \nu B \sum_{h=0}^{H-1} \rho^{h(1-d_m)}.
\end{align*}

For $h\geq 0$, any node $(h,i)\in \mathcal{T}^3$ belongs to $\mathcal{J}_h$. Since the parents of any node $(h,i)\in \mathcal{J}_{h}$, belong to $I_{h-1}$, these nodes are at least $2\nu\rho^{h-1}$-optimal. Then by Assumption~\ref{Ass:HOO}, we have that all points that lie in these cells are at least $4\nu\rho^{h-1}$-optimal. There we can get
\begin{align}\label{eq:rn3}
    \Expectation{R_{N,3}} \leq b\sum_{h=1}^{H} 4\nu\rho^{h-1}\sum_{i:(h,i)\in\mathcal{J}_h} \Expectation{t_{h,i}(N)},
\end{align}

where $t_{h,i}(N)$ is the number of mini-batches where a descendant of node $(h,i)$ is played up to and including round $N$.

The following lemma  bounds the expected number of times  the algorithm visits a suboptimal node $(h,i)$ at the end of round $n$, in terms of the  smoothness parameters $(\nu,\rho)$, batch size $\batch$, and the sub-optimality gap $\Delta_{h,i}$. This bound is  used in the regret bound for HOO-MB. 

\begin{lemma}
\label{l:4}
For any  $n > \batch +1$, and any sub-optimal node $(h,i)$ with $\Delta_{h,i}>\nu\rho^h$:
\begin{align*}
    \mathbb{E}[t_{h,i}(n)] \leq \dfrac{8\sigma^2\log{(\lfloor \dfrac{n-1}{\batch}\rfloor +1)}}{\batch(\Delta_{h,i}-\nu\rho^h)^2}+4.
\end{align*}
\end{lemma}
A sub-optimal node $(h,i)$ might be visited if either the $U_{h,i}$ is greater than $f^*$ or the $U_{h,i*}$ is underestimated. 
This lemma combines these two cases. let's skip the proof of Lemma~\ref{l:4} for now (see for detailed proof in Appendix~\ref{appendix:proof-l}). 

Using Lemma~\ref{l:4}, inequality (\ref{eq:rn3}) can be written as:
\begin{align*}
    \Expectation{R_{N,3}} \leq b\sum_{h=1}^{H} 4\nu\rho^{h-1}|\mathcal{J}_h|\big( \dfrac{8\sigma^2\log{(L(N)+1)}}{b\nu^2\rho^{2h}} + 4 \big).
\end{align*}
Using the fact that $|\mathcal{J}_h| \leq 2|I_{h-1}|$, we have:

\begin{align*}
    \Expectation{R_{N,3}} \leq 8b\nu B \sum_{h=1}^{H} \rho^{(h-1)(1-d_m)} \big( \dfrac{8\sigma^2\log{(L(N)+1)}}{b\nu^2\rho^{2h}} + 4 \big).
\end{align*}
\\

Now, putting the obtained bounds together, we get

\begin{align*}
    \Expectation{R_N} &\leq 4\nu\rho^HN + 4b\nu B \sum_{h=0}^{H-1} \rho^{h(1-d_m)}\\
    & + 8b\nu B\sum_{h=1}^{H} \rho^{(h-1)(1-d_m)}\bigg( \dfrac{8\sigma^2\log{(L(N) +1)}}{b\nu^2\rho^{2h}} + 4\bigg)
\end{align*}
\begin{align}\label{eq:bound1}
    & = O\bigg(\rho^HN + B \big(log{(L(N) +1)} + b \big)\rho^{-H(1+d_m)}\bigg).
\end{align}
Now choosing $H$ such that $\rho^{H} = O\bigg(B\dfrac{log{(L(N)+1)} + b}{N}\bigg)$ and using $L(N)=\lfloor \dfrac{N-1}{b} \rfloor$, minimizes the bound in~(\ref{eq:bound1}) as follows:

\begin{equation}
    \Expectation{R_N} = O\bigg( BN^{\frac{d_m+1}{d_m+2}}\ \bigg(\log (\lfloor \dfrac{N-1}{b} \rfloor + 1 ) + b \bigg)^{\frac{1}{d_m+2}}\bigg)   
\end{equation}

Using the Remark $1$ from~\cite{bubeck2011x}, we can relate the simple regret $S_N$ and the cumulative regret $R_N$ and obtain the following:
$$\mathbb{E}(S_N)=O\bigg(\ \bigg(\frac{B\log (\lfloor \dfrac{N-1}{b} \rfloor + 1 ) + b}{N} \bigg)^{\frac{1}{d_m+2}}\bigg),$$ which concludes the proof of the Theorem~\ref{Th:regret_HOO_batch}.

\section{Appendix: Proof of Lemma~\ref{l:4}}
\label{appendix:proof-l}
Let $x_i \in \mathcal{X}$ be the point returned by Algorithm~\ref{Al:HOO_batch} at round $i$. In Algorithm~\ref{Al:HOO_batch}, once a node is chosen, it is queried for $\batch$ times. Let $L(n) = \left\lfloor \dfrac{n-1}{\batch} \right\rfloor$, then at round $n$, there are $L(n)+1$ mini-batches. Denote by $(h_j, i_j)$ the node chosen by HOO-MB in the $j$-th mini-batch, where $j = 1,\dots, L(n) +1$. Let $z_j$ denotes the arm played in the $j$-th mini-batch, where $z_1 = x_1 = \dots = x_{\batch},\ z_2= x_{\batch+1} = \dots = x_{2\batch}, \dots$. And $y_{j,k}$ denotes the $k$-th observation in the $j$-th mini-batch. 

We begin the proof of Lemma~\ref{l:4}, with a lemma (Lemma $14$ from~\cite{bubeck2011x}) which can be adapted to the Algorithm~\ref{Al:HOO_batch} as follows:

\begin{lemma}\label{l:l14}
Let $(h,i)$ be a sub-optimal node, i.e. $\Delta_{h,i} > 0$. Let $0 \leq l \leq h-1$ be the largest depth such that $(l,i^*_l)$ is on the path from the root $(0,1)$ to $(h,i)$, where $i^*_l$ is the optimal nodes at depth $l$. Then for all integers $u \geq 0$, we have
\begin{align*}\label{eq:u}
    \mathbb{E}[t_{h,i}(n)] \leq u &\sum_{j=u+1}^{L(n) + 1} \mathbb{P} \bigg\{[U_{s,i^*_s}(j) \leq f^*\ for\ some\  \\
    &\hspace{2cm} s\in\{l+1,...,j-1\}] \ \\
    &\hspace{1cm} or \ [\ t_{h,i}(j)>u\ and\ U_{h,i}(j)>f^*]\bigg\}.
\end{align*}
\end{lemma}
See ~\cite{bubeck2011x} for proof of Lemma~\ref{l:l14}.\\

In order to come up with a bound for $\mathbb{E}[t_{h,i}(n)]$, we will use the following two lemmas (Lemma~\ref{l:l15} and Lemma~\ref{l:l16}). Lemma~\ref{l:l15} gives a bound for the first term, and Lemma~\ref{l:l16} gives a bound for the second term.

\begin{lemma}\label{l:l15}
Let Assumption~\ref{Ass:HOO} hold. Then for all optimal nodes $(h,i)$ and all integers $n \geq b$,

\begin{align*}
    \mathbb{P}\{ U_{h,i}(n) \leq f^* \} \leq (L(n) +1)^{-3}. 
\end{align*}
\end{lemma}
\begin{proof}
For the cases that $(h,i)$ was not chosen in the first $n$ rounds, by convention, $U_{h,i}(n)=+\infty$. Therefore, We focus on the event $\{t_{h,i}(n) \geq 1\}$:

\begin{align*}
    \mathbb{P} \{ &U_{h,i}(n) \leq f^*\ and\ t_{h,i}(n) \geq 1 \}\\
    &= \mathbb{P}\big\{ \hat{f}_{h,i}(n) + \sqrt{\dfrac{2\sigma^2\log{(L(n) +1)}}{bt_{h,i}(n)}} + \nu\rho^h \leq f^*\ \\
    &\hspace{5cm} \hbox{and}\ t_{h,i}(n) \geq 1 \big\}\\
    & = \mathbb{P}\big\{ t_{h,i}(n)\hat{f}_{h,i}(n) + t_{h,i}(n)(\nu\rho^h-f^*) \leq\\
    &\hspace{0.5cm}-\sqrt{\dfrac{2\sigma^2t_{h,i}(n)\log{(L(n) +1)}}{b}}\ \hbox{and}\ t_{h,i}(n) \geq 1 \big\}\\
    &= \mathbb{P}\big\{\sum_{l=1}^{L(n) +1} \big( \dfrac{\sum_{j=1}^{b} y_{l,j}}{b} -f(z_l)\big) \mathbb{I}_{\{(h_l,i_l)\in\mathcal{C}(h,i)\}}\\
    &\hspace{1cm} + \sum_{l=1}^{L(n) +1} (f(z_l) + \nu\rho^h-f^*)\mathbb{I}_{\{(h_l,i_l)\in\mathcal{C}(h,i)\}}\\
    &\hspace{1cm} \leq -\sqrt{\dfrac{2\sigma^2t_{h,i}(n)\log{(L(n) +1)}}{b}}\ \\
    & \hspace{5cm} \hbox{and}\ t_{h,i}(n) \geq 1 \big\},
\end{align*}
where $z_l$ is the point chosen in the $l$-th mini-batch, and $y_{l,1},y_{l,2},...,y_{l,k}$ are $k$ observations at point $z_l$, and $\mathcal{C}(h,i)$ denotes the descendants of the node (h,i). According to the Assumption~\ref{Ass:HOO} with $c=0$, $(f(z_l) + \nu\rho^h-f^*)\mathbb{I}_{\{(h_l,i_l)\in\mathcal{C}(h,i)\}} \geq 0$. Therefore we can write:
\begin{align*}
    & \mathbb{P} \{ U_{h,i}(n) \leq f^*\ \hbox{and}\ t_{h,i}(n) \geq 1 \}\\
    & \leq \mathbb{P}\big\{\sum_{l=1}^{L(n) +1} \big( f(z_l) - \dfrac{\sum_{j=1}^{b} y_{l,j}}{b}\big)\mathbb{I}_{\{(h_l,i_l)\in\mathcal{C}(h,i)\}}\\
    &\hspace{1cm} \geq \sqrt{\dfrac{2\sigma^2t_{h,i}(n)\log{(L(n) +1)}}{b}}\ \hbox{and}\ t_{h,i}(n) \geq 1 \big\}. 
\end{align*}
Since $y_{l,j}$ is assumed to be $\sigma^2$-sub-Gaussian, $\dfrac{\sum_{j=1}^{b} y_{l,j}}{b}$ is $\dfrac{\sigma^2}{b}$-sub-Gaussian. Similar to the approach in proof of Lemma $15$ in~\cite{bubeck2011x}, we can take care the last inequality with union bound, optional skipping and the Hoeffding-Azuma inequality and obtain the following:
\begin{align*}
    \mathbb{P} & \{ U_{h,i}(n) \leq f^*\ and\ t_{h,i}(n) \geq 1 \} \leq (L(n) +1)^{-3},
\end{align*}
where we conclude the proof of lemma~\ref{l:l15}.
\hfill $\blacksquare$
\end{proof}

\begin{lemma}\label{l:l16}
For all integers $j \leq n$, all sub-optimal nodes $(h,i)$ such that $\Delta_{h,i} > \nu \rho^h$, and all integers $u \geq 1$ such that $$u\geq \dfrac{8\sigma^2\log{(L(n) +1)}}{b(\Delta_{h,i}-\nu\rho^h)^2},$$ one has $$\mathbb{P}\{U_{h,i}(j) >f^*\ and\ t_{h,i}(j)>u \} \leq (L(j) +1)(L(n) +1)^{-4}.$$
\end{lemma}
\begin{proof} The $u$ in the statement satisfies $$\dfrac{\Delta_{h,i} - \nu \rho^h}{2} \geq \sqrt{\dfrac{2\sigma^2\log{(L(n) +1)}}{bu}},$$ thus $\sqrt{\dfrac{2\sigma^2\log{(L(j) +1)}}{bu}} + \nu \rho^h \leq \dfrac{\Delta_{h,i} + \nu \rho^h}{2}.$ Therefore,
\begin{align*}
    \mathbb{P} & \{U_{h,i}(j) >f^*\ and\ t_{h,i}(j)>u \}\\
    & = \mathbb{P}\big\{ \hat{f}_{h,i}(j) + \sqrt{\dfrac{2\sigma^2\log{(L(j) +1)}}{b\ t_{h,i}(j)}} + \nu\rho^h > f^*_{h,i} + \Delta_{h,i}\ \\
    &\hspace{6cm} \hbox{and}\ t_{h,i}(j)>u  \big\}\\
    & \leq \mathbb{P}\big\{ t_{h,i}(j)(\hat{f}_{h,i}(j) - f^*_{h,i}) > \dfrac{\Delta_{h,i}-\nu\rho^h}{2}t_{h,i}(j) \ \\
    & \hspace{6cm} \hbox{and}\ t_{h,i}(j)>u  \big\}.
\end{align*}
Now similar to the approach in proof of Lemma $16$ in~\cite{bubeck2011x} we can take care of the last inequality with union bound, optional skipping and the Hoeffding-Azuma inequality and obtain the following: $$\mathbb{P}\{U_{h,i}(j) >f^*\ and\ t_{h,i}(j)>u \} \leq (L(j) +1)(L(n) +1)^{-4},$$
where we conclude the proof of lemma~\ref{l:l16}.
\hfill $\blacksquare$
\end{proof}

Now combining the results of Lemma~\ref{l:l14}, Lemma~\ref{l:l15} and Lemma~\ref{l:l16} we obtain the following for sub-optimal nodes with $\Delta_{h,i}>\nu\rho^h$:

\begin{align*}
    \mathbb{E}[t_{h,i}(n)] \leq &\dfrac{8\sigma^2\log{(L(n) +1)}}{b(\Delta_{h,i}-\nu\rho^h)^2} + 1 \\
    & + \sum_{j=u+1}^{L(n) + 1} \bigg(( L(j) +1)(L(n) +1)^{-4} \\
    &\hspace{2cm}+ \sum_{s=1}^{L(j)} (L(j) +1)^{-3} \bigg).
\end{align*}
Now similar to the approach in the proof of the Lemma $16$ in~\cite{bubeck2011x} and using $L(n) = \lfloor \dfrac{n-1}{b} \rfloor$, we get the following:
\begin{align*}
    \mathbb{E}[t_{h,i}(n)] \leq \dfrac{8\sigma^2\log{(\lfloor \dfrac{n-1}{b} \rfloor +1)}}{b(\Delta_{h,i}-\nu\rho^h)^2} + 4,
\end{align*}
which concludes the proof of Lemma~\ref{l:4}.

\end{document}